\documentclass[arxiv]{colt2026}

\usepackage{amssymb}
\usepackage{xcolor}
\usepackage{algorithm,caption}
\usepackage{algorithmic}
\usepackage{cleveref}
\usepackage{graphicx}
\usepackage{natbib}
\usepackage{framed}
\usepackage{times}
\usepackage{mdframed}
\usepackage{xcolor}
\usepackage{tikz}
\usetikzlibrary{arrows.meta,positioning}

\newtheorem{question}[theorem]{Question}

\newenvironment{restatedtheorem}[1]
  {\par\noindent\textbf{Theorem~\ref{#1}.}\itshape}
  {\par}

\newenvironment{restatedproposition}[1]
  {\par\noindent\textbf{Proposition~\ref{#1}.}\itshape}
  {\par}
\newtheorem*{theorem*}{Theorem}
\newtheorem*{lemma*}{Lemma}
\newtheorem*{proposition*}{Proposition}
\newtheorem*{corollary*}{Corollary}
\newtheorem*{claim*}{Claim}
\newtheorem*{definition*}{Definition}
\newtheorem*{assumption*}{Assumption}
\newtheorem*{question*}{Question}
\newtheorem*{remark*}{Remark}
\newtheorem*{example*}{Example}
\newcounter{examplectr}

\newcommand{\N}{\mathbb{N}}

\newcommand{\X}{\mathcal{X}}
\newcommand{\Y}{\mathcal{Y}}
\newcommand{\F}{\mathcal{F}}
\newcommand{\A}{\mathcal{A}}
\newcommand{\D}{\mathcal{D}}
\newcommand{\G}{\mathcal{G}}
\newcommand{\W}{\mathcal{W}}
\newcommand{\Hc}{\mathcal{H}}
\newcommand{\Pc}{\mathcal{P}}
\newcommand{\Ber}{Ber}

\newcommand{\E}{\mathbb{E}}

\newcommand{\eps}{\varepsilon}

\newcommand{\VC}{\mathrm{VC}}
\newcommand{\LVC}{\mathrm{LVC}} 

\newcommand{\TV}{\mathrm{TV}}

\newcommand{\KL}{\mathrm{KL}}

\newcommand{\Unif}{\mathrm{Unif}}
\newcommand{\diag}{\mathrm{diag}}

\newcommand{\binomle}[2]{\binom{#1}{\le #2}}

\title[Uniform Laws of Large Numbers in Product Spaces]{Uniform Laws of Large Numbers in Product Spaces}
\coltauthor{%
 \Name{Ron Holzman} \Email{holzman@technion.ac.il}\\
 \addr Department of Mathematics, Technion
 \AND
 \Name{Shay Moran}\Email{smoran@technion.ac.il}\\
 \addr Departments of Mathematics, Computer Science, and Data and Decision Sciences, Technion and Google Research
 \AND
 \Name{Alexander Shlimovich} \Email{ashlimovich@campus.technion.ac.il}\\
 \addr Department of Mathematics, Technion
}
\begin{document}
\maketitle
\begin{abstract}

Uniform laws of large numbers form a cornerstone of Vapnik--Chervonenkis theory, where they are characterized by the finiteness of the VC dimension. In this work, we study uniform convergence phenomena in \emph{cartesian product spaces}, under assumptions on the underlying distribution that are compatible with the product structure. Specifically, we assume that the distribution is absolutely continuous with respect to the product of its marginals, a condition that captures many natural settings, including product distributions, sparse mixtures of product distributions, distributions with low mutual information, and more.

We show that, under this assumption, a uniform law of large numbers holds for a family of events if and only if the \emph{linear VC dimension} of the family is finite. The linear VC dimension is defined as the maximum size of a shattered set that lies on an \emph{axis-parallel line}, namely, a set of vectors that agree on all but at most one coordinate. This dimension is always at most the classical VC dimension, yet it can be arbitrarily smaller.
For instance, the family of convex sets in $\mathbb{R}^d$ has linear VC dimension~$2$, while its VC dimension is infinite already for $d\ge 2$. Our proofs rely on estimator that departs substantially from the standard empirical mean estimator and exhibits more intricate structure. We show that such deviations {from the standard empirical mean estimator} are unavoidable in this setting. 
Throughout the paper, we propose several open questions, with a particular focus on quantitative sample complexity bounds.

\end{abstract}


\section{Introduction}

A central theme in statistics and learning theory is to understand when empirical averages provide reliable \emph{uniform} approximations to their population counterparts. Such uniform laws of large numbers lie at the heart of statistical learning theory: they underlie generalization guarantees, sample complexity bounds, and the analysis of learning algorithms. Classical results show that, for families of events, uniform convergence is characterized by the Vapnik--Chervonenkis (VC) dimension~\cite{vapnik1971}, while for real-valued function classes it is governed by complexity measures such as Rademacher averages or fat-shattering dimensions~\cite{kearns1994,alon1997,bartlett1998}. These results play a basic role in learning theory and in the study of generalization.

Much of the classical theory treats the underlying distribution as arbitrary and makes no use of additional structure of the domain. In many applications, however, the domain naturally carries a product structure, with data points represented as vectors of features or measurements. In such settings, it is often reasonable to restrict attention to distributions that are compatible with this structure, for instance through independence or weak dependence across coordinates. The focus of this work is to understand how such product structure can be leveraged in the uniform estimation problem, and to identify complexity measures that govern uniform convergence under these structural assumptions.

\subsection{Uniform Estimation}

{We begin by presenting an abstract \emph{uniform estimation} problem.} 
This framework captures a broad range of questions in statistics and learning theory and provides a common language for classical uniform laws of large numbers as well as for the results developed in this work.

\begin{mdframed}[
  backgroundcolor=gray!5,
  linecolor=gray!40,
  linewidth=0.5pt,
  roundcorner=6pt,
  innertopmargin=0.8em,
  innerbottommargin=0.8em
]
\begin{center}
\textbf{Uniform Estimability}
\end{center}
\medskip
\noindent
Let $\X$ be a domain, let $\Pc$ be a family of distributions over $\X$, and let $\F$ be a class of real-valued functions on $\X$.
Under what conditions does there exist an estimator which, given i.i.d.\ samples from an unknown distribution $P \in \Pc$, approximates the expectations
\[
P(F) := \mathbb{E}_{X \sim P}[F(X)]
\]
uniformly and simultaneously for all $F \in \F$?
\end{mdframed}

For simplicity, and in line with much of the learning-theoretic literature, we focus on the case where $\F$ consists of indicator functions of measurable subsets of $\X$. Thus, each $F \in \F$ corresponds to an event, and the quantity of interest is the probability $P(F)$. This setting encompasses the classical uniform convergence problem studied in Vapnik--Chervonenkis theory. While our presentation focuses on events, the ideas and techniques developed in this paper extend naturally to more general classes of real-valued functions.

{From a statistical perspective, the goal is to estimate these probabilities
from finitely many independent samples.}
 An estimation algorithm $\A$ is therefore a mapping
\(
\A : \bigcup_{m=1}^\infty \X^m \to [0,1]^\F.
\)

Given a sample $x_1,\ldots,x_m \in \X$ drawn independently from an unknown distribution $P \in \Pc$, the algorithm outputs estimates $\widehat P_m(F) \in [0,1]$ for the probabilities $P(F)$, simultaneously for all $F \in \F$. A canonical example is the empirical estimator, which assigns
\(\widehat P_m(F) := \frac{1}{m}\sum_{i=1}^m \mathbf{1}[x_i \in F].\)

When it is possible to approximate the probabilities of all events in $\F$ uniformly over all target distributions in $\Pc$, we say that the pair $(\F,\Pc)$ is \emph{uniformly estimable}.
\vspace{-1em}
\begin{definition}[Uniform Estimability]\label{def:unifest}
We say that a pair $(\F, \Pc)$ is \emph{uniformly estimable} if there exist a sample complexity bound $m : (0,1)^2 \to \N$ and an algorithm $\A$ such that for every $\eps, \delta \in (0,1)$ and every distribution $P \in \Pc$, the following holds: if $x_1, \ldots, x_m$ are drawn independently from $P$ for $m \ge m(\eps, \delta)$,
then with probability at least $1 - \delta$,
\begin{equation*}\label{eq:goal}
\forall F \in \F:\quad \bigl| \widehat P_m(F) - P(F) \bigr| \le \eps,
\end{equation*}
where $\widehat P_m = \A(x_1, \ldots, x_m)$ denotes the estimator output by $\A$.
\end{definition}

\vspace{-0.5em}
Before presenting our main results, we illustrate the scope of the uniform estimability framework by revisiting two classical extremes in statistical learning theory. The first corresponds to the case where the family of distributions $\Pc$ is unrestricted, leading to Vapnik--Chervonenkis theory. The second corresponds to the case where the family of events $\F$ is unrestricted, which is intimately related to the \emph{hypothesis selection} problem. These examples delineate the classical boundaries of uniform estimation and serve as reference points for the structured distribution families studied in this work.

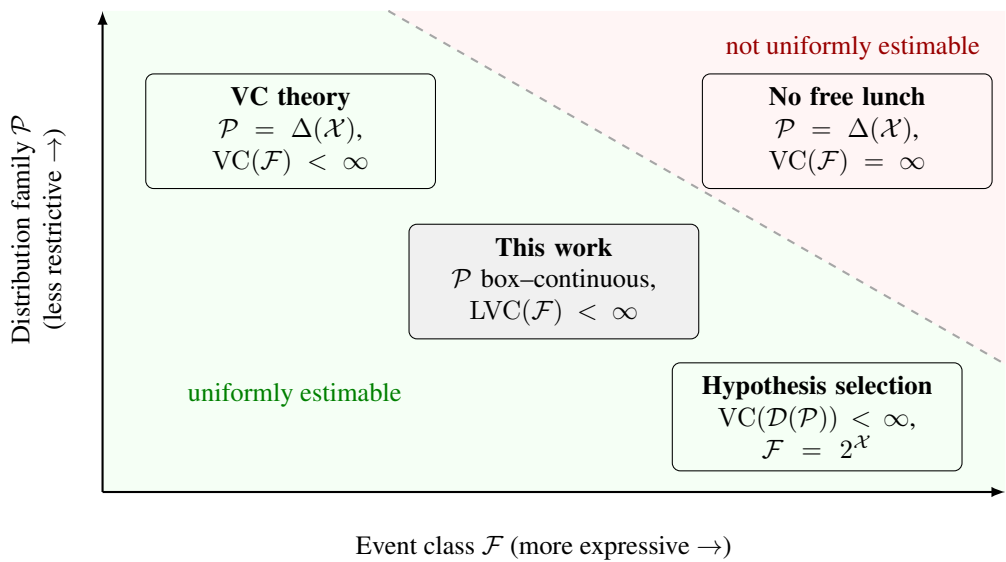
\begin{figure}[t]
\centering
\begin{tikzpicture}[
  font=\small,
  >=latex,
  arrow/.style={->, thick},
  lab/.style={font=\small, align=center},
  tick/.style={font=\footnotesize, align=center},
  box/.style={draw, rounded corners=3pt, align=center, inner sep=5pt, text width=3.5cm},
  work/.style={box, fill=gray!12},
  boundary/.style={dashed, gray!70, line width=0.8pt},
  region/.style={fill=gray!10, draw=gray!60, rounded corners=3pt}
]


\node[lab, anchor=north east] at (8.5,-0.45)
  {Event class $\F$ (more expressive $\rightarrow$)};
\node[lab, rotate=90] at (-0.85,3.5)
  {Distribution family $\Pc$\\ (less restrictive $\rightarrow$)};

\coordinate (A) at (3.8, 6.4);
\coordinate (B) at (12,1.7);
\draw[boundary] (A) -- (B);

\fill[green!4] (0,0) -- (12,0) -- (B) -- (A) -- (0, 6.4) -- cycle;

\fill[red!4] (A) -- (B) -- (12,6.4) -- cycle;

\draw[arrow] (0,0) -- (12,0);
\draw[arrow] (0,0) -- (0,6.4);

\draw[boundary] (A) -- (B);

\node[lab, anchor=north east, text=red!60!black] at (11.8,6.2)
  {not uniformly estimable};

\node[lab, anchor=south west, text=green!50!black] at (1,1)
  {uniformly estimable};

\node[work] at (6,2.8) {
\textbf{This work}\\
$\Pc$ box--continuous,\ $\LVC(\F)<\infty$
};

\node[box] at (2.5,4.8) {
\textbf{VC theory}\\
$\Pc=\Delta(\X)$,\ $\VC(\F)<\infty$
};
\node[box] at (9.9,4.8) {
\textbf{No free lunch}\\
$\Pc=\Delta(\X)$,\\ $\VC(\F)=\infty$
};
\node[box] at (9.5,1) {
\textbf{Hypothesis selection}\\
$\VC(\D(\Pc))<\infty$, \ $\F=2^{\X}$
};

\end{tikzpicture}

\caption{Uniform estimation regimes. The dashed curve separates uniformly and non-uniformly estimable regimes based on the richness of the event class $\F$ and distribution family $\Pc$.}
\label{fig:uniform-estimation-map}
\end{figure}

\subsubsection*{The Family of Distributions $\Pc$ is Unrestricted (VC Theory)}

Consider the classical distribution-free setting in which the family of distributions is unrestricted, $\Pc = \Delta(\X)$.
The central question of Vapnik--Chervonenkis theory is to characterize those (possibly infinite) classes $\F$ for which $(\F, \Delta(\X))$ is uniformly estimable.\footnote{%
Classical VC theory focuses on uniform convergence of the empirical mean estimator and does not explicitly consider arbitrary estimators.
However, the no-free-lunch lower bound underlying the necessity of finite VC dimension applies to \emph{any} estimator, and thus the characterization extends beyond the empirical mean.}
The seminal work of~\cite{vapnik1971} shows that this holds if and only if $\F$ has finite VC dimension $d$.
Moreover, the empirical distribution is minimax--optimal in this setting, achieving the sharp sample complexity
\(
m(\varepsilon,\delta) = \Theta\bigl((d + \log(1/\delta))/\varepsilon^2\bigr)
\)
\citep{blumer1989,talagrand1994}.

\vspace{-1em}
\subsubsection*{The Family of Events $\F$ is Unrestricted (Hypothesis Selection)}

Consider the opposite extreme, in which the family of events is unrestricted, $\F = 2^{\X}$.
In this case, uniform estimation coincides with learning the underlying distribution in total variation distance.

When the family of distributions $\Pc = \{P_1,\ldots,P_n\}$ is finite, this problem admits a classical solution via the hypothesis selection method of~\cite{yatracos1985}.
The method is based on the \emph{Yatracos sets}
\(
D_{i,j} := \{x \in \X : P_i(x) \ge P_j(x)\},
\)
and the associated class $\D(\Pc) := \{D_{i,j} : i,j \in [n]\}$.
These sets satisfy
\[
\TV(P_i,P_j) = \max_{D \in \D} \bigl| P_i(D) - P_j(D) \bigr|
\qquad\text{for all } i,j \in [n].
\]
As a consequence, with
\(m = \Theta\bigl((\log n + \log(1/\delta))/\varepsilon^2\bigr)\)
samples, the empirical distribution $\hat p_m$ uniformly approximates the probabilities of all sets in $\D$.
Selecting a distribution $p_i \in \Pc$ that minimizes 
$\max_{D \in \D} \bigl| P_i(D) - \hat P_m(D) \bigr|$ yields a sound estimator.
More generally, the same approach extends to infinite families $\Pc$ whenever the associated Yatracos class $\D(\Pc)$ has finite VC dimension~$d$, in which case the sample complexity scales as
\(m = \Omega\bigl((d + \log(1/\delta))/\varepsilon^2\bigr).\) 

In contrast to the VC setting discussed above, a general characterization of those infinite families of distributions $\Pc$ for which the pair $(\F = 2^{\X}, \Pc)$ is uniformly estimable is not known.
Moreover, it is known that VC-type combinatorial characterizations are impossible in this regime~\citep{lechner2024}.

\vspace{-0.5em}
\paragraph{Roadmap.}
In Section~\ref{sec:main}, we state our main result {in qualitative form.}
Within this section, Section~\ref{sec:examples} presents a collection of examples
illustrating the scope and applicability of our result.
{In Section~\ref{subsec:quantitative-bounds}, we discuss the quantitative
sample complexity bounds that follow from our analysis and formulate several
open questions aimed at closing the remaining gaps.}
In Section~\ref{sec:technicaloverview}, we provide a high-level overview of the
proof and describe the resulting estimation procedure.
Full proofs are deferred to the appendix.

\vspace{-1em}
\section{Main Result}\label{sec:main}

We now turn to uniform estimation on product domains \(\X = \W_1 \times \cdots \times \W_d.\)
To state our main result, we introduce two notions.
The first is a distributional regularity condition, termed \emph{uniform box--continuity}.
The second is a combinatorial parameter of $\F$, called the \emph{linear VC dimension}.
With these definitions in place, we then state our main theorem.

\paragraph{Uniform Box--Continuity.}
We begin by recalling a quantitative formulation of absolute continuity.
Let $P$ and $Q$ be probability distributions on a measurable space $(\X,\Sigma)$.
Recall that $P$ is absolutely continuous with respect to $Q$ if for every
$\alpha>0$ there exists $\beta>0$ such that for every measurable set $E$,
\[
P(E)\ge \alpha \ \Longrightarrow\ Q(E)\ge \beta .
\]
This condition is the usual notion of absolute continuity, commonly denoted in
measure theory by $P\ll Q$. 
We now specialize this notion to product spaces.
Let $\X = \W_1 \times \cdots \times \W_d$, and let $P \in \Delta(\X)$.
Denote by $P_i$ the marginal of $P$ on $\W_i$, and by
$P_{\square} := \prod_{i=1}^d P_i$ the product of marginals.
\vspace{-0.5em}
\begin{definition}[Uniform box--continuity]
\label{def:uniform-box-cont}
A family of distributions $\Pc \subseteq \Delta(\X)$ is \emph{uniformly box--continuous}
if for every $\alpha>0$ there exists $\beta >0$ such that for every $P \in \Pc$ and every measurable set $E \subseteq \X$, \(P(E) \ge \alpha
\ \Longrightarrow\
P_{\square}(E) \ge \beta\).
\end{definition}
\vspace{-0.3em}
Whenever a family $\Pc$ is uniformly box--continuous, we refer to any function
$\beta \colon (0,1] \to (0,1]$ satisfying Definition~\ref{def:uniform-box-cont}
as a \emph{modulus of box--continuity} for $\Pc$. 
We denote by $\Pc_{\beta}$ the family of all distributions that are box--continuous
with modulus $\beta$.

This condition is analogous to uniform continuity in analysis: the same
$\alpha \mapsto \beta(\alpha)$ relationship applies uniformly to all distributions
in the family $\Pc$.
For example, the family of all product distributions on $\X$ is uniformly box--continuous,
with $\beta(\alpha) = \alpha$.
We present additional examples of uniformly box--continuous families after stating the main result.

\paragraph{Linear VC Dimension.}
We now introduce the combinatorial parameter governing uniform estimation in product spaces.
A set $L \subseteq \X$ is called an \emph{axis--parallel line} if there exists an
index $i \in [d]$ and elements $w_j \in \W_j$ for all $j \neq i$ such that

\vspace{-0.5em}
\[
L = \bigl\{ (x_1,\ldots,x_d) \in \X : x_j = w_j \text{ for all } j \neq i \bigr\}.
\]

\vspace{-0.3em}
 A set $S \subseteq \X$ is called \emph{colinear} if it is contained in some 
 axis-parallel line. Equivalently, $S$ is a set of points in $\X$ that agree on all but one coordinate.

\vspace{-0.5em}
\begin{definition}[Linear VC dimension]
\label{def:lvc}
The \emph{linear VC dimension} of a class $\F \subseteq 2^{\X}$, denoted $\LVC(\F)$, is the
largest integer $k$ for which there exists a colinear set $S \subseteq \X$ of size $k$
that is \emph{shattered} by $\F$, meaning that for every labeling $S \to \{0,1\}$,
there exists a set $F \in \F$ that realizes this labeling on $S$.
If no such finite $k$ exists, we set $\LVC(\F) = \infty$.
\end{definition}

\vspace{-1em}
\begin{mdframed}[
  backgroundcolor=gray!10,
  linecolor=gray!50,
  linewidth=0.6pt,
  roundcorner=6pt,
  innertopmargin=1em,
  innerbottommargin=1em
]
\begin{theorem}[Main Result]
\label{thm:characterization}
Let $\X = \W_1 \times \cdots \times \W_d$ be a product measurable space, and let
$\F \subseteq 2^{\X}$ be a family of events.
Then, the following are equivalent:
\begin{enumerate}
\item $\LVC(\F) < \infty$.
\item For every uniformly box--continuous family of distributions
$\Pc \subseteq \Delta(\X)$, the pair $(\F,\Pc)$ is uniformly estimable.
\end{enumerate}
\end{theorem}
\end{mdframed}

The necessity of finite linear VC dimension is relatively straightforward and follows from standard no--free--lunch lower bounds for VC classes.
The more substantive part of the theorem is the converse direction, showing that finite linear VC dimension is also sufficient for uniform estimability under uniformly box--continuous distributions.
In this sense, Theorem~\ref{thm:characterization} identifies linear VC dimension as the only obstruction to uniform estimation in product spaces within this regime.
Establishing sufficiency requires a combination of techniques, including packing arguments and the construction of grid--based estimators that explicitly exploit the underlying product structure.
We provide a high--level overview of the proof strategy in Section~\ref{sec:technicaloverview}.

The special case in which the family $\Pc$ consists solely of product distributions has been studied in several prior works
\citep{cai2017,guo2020,harms2022,livni2019,coregliano2024}.
In particular, the results of \citet{cai2017}, \citet{livni2019}, and \citet{coregliano2024} imply our characterization in this regime.
The conclusion of \citet{cai2017} is weaker, as uniform estimation there is implied by the finiteness of the VC dimension of the coordinatewise projections of $\F$, a parameter that is always at least as large as the linear VC dimension 
and can be infinite even when $\LVC(\F)$ is finite.
By contrast, \citet{livni2019} were the first to introduce a combinatorial dimension essentially equivalent to the linear VC dimension and use it to characterize uniform estimation in the symmetric i.i.d.\ setting; this perspective was further developed in~\citet{coregliano2024}, which replaces the i.i.d.\ assumption by weaker symmetry notions related to \emph{exchangeability}.

Unlike these works, our results do not rely on symmetry assumptions or exact independence, and apply to any distribution families exhibiting weak independence as formalized by uniform box--continuity.
For example, our framework covers finite mixtures of product distributions as well as distributions with low mutual information, such as non-degenerate multivariate Gaussians, which fall outside the scope of previous work.
Additional examples illustrating this broader regime are discussed in the following section.

\vspace{-1em}
\subsection{Examples}\label{sec:examples}

In this section we present several basic examples that illustrate the scope of
Theorem~\ref{thm:characterization} and aim to provide intuition for the notions of
linear VC dimension and uniform box-continuity.
Examples~\ref{ex:mixtures} and~\ref{ex:total-correlation} present natural
distribution families satisfying uniform box--continuity.
Example~\ref{ex:convex} shows a natural class with infinite VC dimension
and finite linear VC dimension.
Example~\ref{ex:perm-comparison} highlights a setting in which the classical
empirical estimator fails, motivating the need for estimators that explicitly
exploit product structure.

\refstepcounter{examplectr}
\paragraph{Example~\theexamplectr\ (Mixtures of product distributions).}
\label{ex:mixtures}
We begin with a basic example of uniform box-continuity beyond exact product distributions, given by finite mixtures of product distributions.
Establishing uniform box-continuity in this case already requires some care.

\vspace{-0.5em}
\begin{proposition}[Mixtures of product distributions]
\label{prop:mixture-box}
Let $P$ be a probability distribution on $\X=\W_1\times\cdots\times\W_d$ that can
be written as a mixture of at most $k$ product distributions.
Then $P$ is uniformly box--continuous with modulus

\vspace{-0.5em}
\[
\beta(\alpha)
\;=\;
\frac{\alpha^d}{(k-1+\alpha)^{\,d-1}}.
\]
\end{proposition}

\vspace{-0.3em}
{We prove Proposition~\ref{prop:mixture-box} in Appendix~\ref{sec:proof-mixture}, where we also show that the modulus is almost optimal, as $\beta(\alpha)\le \alpha^d/(k-1)^{\,d-1}$ is unavoidable.}

Consequently, all mixtures with at most $k$ components share a common modulus
$\beta(\alpha)=\alpha^d/(k-1 +\alpha)^{\,d-1}$ and form a uniformly box--continuous family. In the special case $k=1$, this reduces to the setting of product distributions,
for which $P=P_{\square}$ and uniform box--continuity holds with the identity
modulus $\beta(\alpha)=\alpha$.

\refstepcounter{examplectr}
\paragraph{Example~\theexamplectr\ (Bounded mutual information).}
\label{ex:total-correlation}
Another natural source of uniform box--continuity arises from
information--theoretic constraints.
Recall that for a probability measure $P$ on
$\X=\W_1\times\cdots\times\W_d$, the box projection $P_{\square}$ denotes the
product distribution obtained from the one--dimensional marginals of $P$.
The \emph{total correlation} (also known as \emph{multi--information}) is defined by
\(\mathrm{TC}(P)\;:=\;\KL\!\left(P\,\middle\|\,P_{\square}\right),
\)
where \(\KL\) denotes the Kullback-Leibler divergence.
For $d=2$, the total correlation coincides with the usual mutual information.

\vspace{-0.5em}
\begin{proposition}[Bounded total correlation implies uniform box--continuity]
\label{prop:tc-box}
Fix $C\ge 0$ and let $\Pc_C := \{P:\ \mathrm{TC}(P)\le C\}$.
Then the family $\Pc_C$ is uniformly box--continuous with modulus

\vspace{-0.5em}
\[
\beta(\alpha)
\;=\;
\exp\!\left(\frac{-H(\alpha)-C}{\alpha}\right),
\]
where $H(\alpha) := -\alpha \log \alpha - (1-\alpha)\log(1-\alpha)$ denotes the binary entropy function.
In particular, every $P\in\Pc_C$ is box--continuous with this modulus.\footnote{When $C=0$, one has $\mathrm{TC}(P)=\KL(P\,\|\,P_{\square})=0$, which forces $P=P_{\square}$ and hence $\Pc_0$ coincides with the family of product distributions. In this case the optimal modulus of box--continuity is $\beta(\alpha)=\alpha$. Moreover, plugging $C=0$ into the displayed formula gives $\beta(\alpha)=\exp(-H(\alpha)/\alpha)=\tfrac{\alpha}{e}\,(1+o(1))$ as $\alpha\to 0$. The entropy--based bound stated here remains valid for $C=0$, but is conservative and does not capture this sharp behavior.}
\end{proposition}

\vspace{-0.5em}
Consequently, for every fixed $C>0$, the class of distributions on $\X$ whose
total correlation is at most $C$ is uniformly box–continuous.
The proof of Proposition~\ref{prop:tc-box} is given in Appendix~\ref{sec:proof-bounded-info}.


\smallskip
\noindent
\emph{Gaussian case.}
Multivariate Gaussian distributions provide a concrete illustration of this phenomenon.
Let $X=(X_1,\ldots,X_d)$ be a Gaussian random vector in $\mathbb R^d$ with
covariance matrix~$\Sigma$.
In this case, the total correlation admits the explicit closed--form expression

\vspace{-0.5em}
\[
\mathrm{TC}(X)
\;=\;
\frac12 \log\!\left(\frac{\det(\operatorname{diag}\Sigma)}{\det(\Sigma)}\right),
\]

where $\Sigma$ is the covariance matrix of $X$ and $\operatorname{diag}\Sigma$
is obtained by zeroing out all off--diagonal entries of $\Sigma$; see, e.g.,
\citet[Section~8.3]{cover2012}. This quantity measures the deviation from independence in terms of the volume ratio between the product of marginals and the joint distribution.
In particular, any family of Gaussian distributions for which this quantity is
uniformly bounded has bounded total correlation, and therefore falls within the
scope of Theorem~\ref{thm:characterization}.

In the bivariate case $d=2$, this expression simplifies to
$\mathrm{TC}(X)
=
\tfrac12 \log\!\left(\frac{1}{1-\rho^2}\right),$
where $\rho$ is the correlation coefficient between $X_1$ and $X_2$.
Thus, uniform box--continuity holds for bivariate Gaussians as long as
$\rho^2$ is bounded away from $1$. 


\vspace{-0.4em}
\refstepcounter{examplectr}
\paragraph{Example~\theexamplectr\ (Convex sets in $\mathbb R^d$).}\label{ex:convex}
The power of Theorem~\ref{thm:characterization} is most clearly seen when combined with hypothesis classes that lie far beyond the reach of classical VC theory. Let $\X=\mathbb R^d$ and let $\F$ be the family of all convex subsets of $\mathbb R^d$. Then

\vspace{-0.5em}
\[
\VC(\F)=\infty
\qquad\text{but}\qquad
\LVC(\F)=2.
\]

\vspace{-0.3em}
To see that the classical VC dimension of $\F$ is infinite, note that any finite set of points in convex position in $\mathbb{R}^d$ is shattered by convex sets.
When $d \ge 2$, there exist arbitrarily large such sets (for example, points on a circle), and hence $\VC(\F)=\infty$.
On the other hand, the linear VC dimension satisfies $\LVC(\F)=2$, since restricting convex sets to any line in $\mathbb{R}^d$ yields the class of intervals, which has VC dimension~2.
Consequently, while convex sets are not uniformly estimable under arbitrary distributions, Theorem~\ref{thm:characterization} implies that $(\F,\Pc)$ is uniformly estimable for every uniformly box--continuous family $\Pc$.

In fact, to bound the linear VC dimension, it suffices to control the behavior of the class on axis--parallel lines, rather than on all affine lines.
For example, consider the class $\F_{\mathrm{stair}}$ of \emph{stair--convex} subsets of $\mathbb{R}^d$, where a set is stair--convex if for every two points that differ in a single coordinate and belong to the set, the entire axis--parallel line segment between them is also contained in the set.
Restricting any stair--convex set to a coordinate axis yields an interval, and hence $\LVC(\F_{\mathrm{stair}})=2$.

\vspace{-0.4em}
\refstepcounter{examplectr}
\paragraph{Example~\theexamplectr\ (Permutation graphs: why new machinery is needed).}
\label{ex:perm-comparison}
Let $\X=[n]\times[n]$, and let $\F=\{F_\pi : \pi\in S_n\}$ be the family of
\emph{permutation graphs}, where each $F_\pi$ corresponds to a permutation
$\pi:[n]\to[n]$ and is defined by

\vspace{-0.5em}
\[
F_\pi := \{(i,\pi(i)) : i\in[n]\}.
\]

\vspace{-0.4em}
Thus, each set in $\F$ contains exactly one point in every row and every column and hence $\LVC(\F)=1$.
By Theorem~\ref{thm:characterization}, $\F$ is therefore uniformly estimable with respect to any uniformly box--continuous family of distributions.

Nevertheless, the most natural estimator - namely, the empirical mean - fails even
in this simple setting.
To see this, consider the uniform distribution on $\X=[n]\times[n]$, which is a
product distribution and hence uniformly box--continuous with
$\beta(\alpha)=\alpha$.
Suppose we draw $m \ll \sqrt{n}$ samples from this distribution.
By the birthday paradox, with high probability no two samples share the same row
and no two share the same column.
Consequently, there exists a permutation $\pi$ such that all sampled points lie
in $F_\pi$.
For the empirical estimator, this yields $\widehat P_m(F_\pi)=1$, whereas the true
probability satisfies $P(F_\pi)=1/n$.

Since $n$ can be taken arbitrarily large, this shows that the empirical mean does
not uniformly estimate $\F$, despite the fact that $\LVC(\F)=1$.
This example highlights the need for estimators that explicitly exploit product
structure, as predicted by our main result.
Formal proofs of the claims above are given in
Section~\ref{sec:proof-permutation}.

\vspace{-1em}
\subsection{Quantitative Bounds}
\label{subsec:quantitative-bounds}

\vspace{-0.3em}
Beyond the qualitative characterization provided by
Theorem~\ref{thm:characterization}, it is natural to ask for
\emph{quantitative} guarantees.
In this subsection we discuss the sample complexity of uniform estimation and
state explicit upper and lower bounds on the sample size needed to achieve uniform estimation.
\vspace{-1em}
\subsubsection{Upper and lower bounds}

\vspace{-0.3em}
To state our quantitative bounds, we use the standard sample--complexity
notation for uniform estimation.
Let $m_{\F,\Pc}(\eps,\delta)$ denote the smallest $m\in\mathbb N$ for which there
exists an estimator $\widehat P_{{m}}$, as in Definition~\ref{def:unifest}, such that

\vspace{-0.5em}
\[
\Pr_{{S\sim p^{m}}}\Bigl[\sup_{F\in\F}\bigl|\widehat P_{m}(S)(F)-P(F)\bigr|\le \eps\Bigr]
\;\ge\; 1-\delta
\qquad\text{for every }P\in\Pc .
\]

\vspace{-0.1em}
Our proof of Theorem~\ref{thm:characterization} is constructive and therefore
yields explicit bounds on $m_{\F,\Pc}(\eps,\delta)$.
In particular, we obtain the following general lower and upper bounds.
Let $g:=\LVC(\F)$.
Our lower bound shows that even when $\Pc$ is the class of product
distributions one necessarily has

\vspace{-0.3em}
\begin{equation}\label{eq:qb-lb}
m_{\F,\Pc}(\eps,\delta)
\;\ge\;
\Omega\!\left(\frac{g+\log(1/\delta)}{\eps^2}\right).
\end{equation}

On the other hand, our main estimator yields the following general upper bound:
\begin{equation}\label{eq:qb-ub}
m_{\F,\Pc}(\eps,\delta)
\;\le\;
 \tilde{O}\!\left(
\frac{{{C_0}^d} \cdot d^{2(d-1)}}{\eps^2\,\beta(\eps/2)^{2(d-1)}}
\; g^{d}\;
\Bigl(\log\tfrac{1}{\delta}\Bigr)^{d-1}
\right),
\end{equation}
where {${C_0} > 0$ is a universal constant},  $d$ is the \emph{width} of the product space $\X=\W_1\times\cdots\times\W_d$ and $\beta$ is the modulus of uniform box--continuity of $\Pc$; namely, for every measurable event $E\subseteq\X$, if $P(E)\ge \eps/2$ then $P_{\square}(E)\ge \beta(\eps/2)$. This is proved in Appendix~\ref{sec:proof-main}; see Lemma~\ref{lem:m-bound-prop1}.

\vspace{-0.3em}
\paragraph{Special case: product distributions.}
When $\Pc$ consists solely of product distributions, a sharper upper bound on the
sample complexity can be obtained:

\vspace{-0.5em}
\[
m_{\F,\Pc}(\varepsilon,\delta)
\;\le\;
O\!\left(\frac{d^2}{\varepsilon^2}\Bigl(g+\log\tfrac{1}{\delta}\Bigr)\right).
\]

\vspace{-0.1em}
This bound is also recovered as a special case of our analysis
(see Remark~\ref{rem:product-high-prob} in Appendix~\ref{sec:proof-main}),
and in the product case it can be achieved using a similar estimator that exploits the exact
product structure, without the need to handle uniform box--continuity.
The same bound can be derived from the arguments of
\citet{cai2017,livni2019}, while one can obtain a conceptually similar guarantee using
\citet{coregliano2024}.

\smallskip
A comparison of~\eqref{eq:qb-lb} and~\eqref{eq:qb-ub} reveals a substantial gap
between our current lower and upper bounds.
Closing this gap leads to the following quantitative question.

\vspace{-0.5em}
\begin{question}[Optimal dependence]
\label{oq:tightness-d}
Let $\Pc$ be a uniformly box--continuous family with modulus $\beta$, and let
$\F$ satisfy $g:=\LVC(\F)$.
What are the tightest bounds on the sample
complexity $m_{\F,\Pc}(\eps,\delta)$ that can be stated in terms of
$d,g,\eps,\delta$, and $\beta$?
\end{question}

\vspace{-0.5em}
Among the various parameter gaps, the dependence on the width $d$ stands out:
our lower bound does not involve $d$, whereas the current upper bound has a
super-exponential $\exp(\Theta(d\log d))$ dependence on $d$.

\vspace{-0.3em}
\paragraph{Tightness of the lower bound.}
We first observe that the lower bound~\eqref{eq:qb-lb} is already tight, up to
constant factors, in its dependence on the parameters $g,\varepsilon,$ and
$\delta$.
This is witnessed by degenerate settings in which the product structure plays no
essential role and the problem effectively reduces to a one--dimensional VC
class.
Specifically, if $\F$ is contained in a single axis--parallel line, then
$\VC(\F)=\LVC(\F)=g$, and uniform estimation reduces to the classical
one--dimensional setting.
In this case, the optimal sample complexity satisfies
$m_{\F,\Pc}(\eps,\delta)=\Theta\!\left(\frac{g+\log(1/\delta)}{\eps^2}\right).$

\vspace{-0.3em}
{\paragraph{Tightness of the upper bound.}
The following example shows that for certain hypothesis classes, a dependence on
the width $d$ is unavoidable, even under product
distributions.}

\refstepcounter{examplectr}
\paragraph{Example~\theexamplectr\ (Unavoidable dependence on the width $d$).}
\label{ex:d-dependence}
Let $\X=\{0,1\}^d$ and $\F=2^{\X}$.
Then $\LVC(\F)=2$, since every axis-parallel line in $\{0,1\}^d$ has size $2$ and
$\F$ restricted to such a line is the full power set.
Moreover, the quantity $\sup_{F\in\F}|\widehat P(F)-P(F)|$
coincides with the total variation distance.
Proposition~\ref{prop:Omega-d-product} shows that for the family of product
measures

\vspace{-0.5em}
\[
\Pc
=\Bigl\{\bigotimes_{i=1}^d \Ber\!\left(\tfrac12+\theta_i\nu\right)
:\theta\in\{\pm1\}^d\Bigr\},
\qquad
\nu = \Theta\!\left(\sqrt{\tfrac{\eps}{d}}\right),
\]

\vspace{-0.3em}
any estimator that is uniformly $\eps$--accurate in total variation must use at
least $m=\Omega(d/\eps)$ samples.

This example shows that bounded $\LVC(\F)$ does not rule out a dependence on the
width $d$.
At the same time, there are settings in which dimension--free rates are
achievable.
Understanding which additional structure of $\F$ distinguishes these regimes is
a natural next question.

\vspace{-0.5em}
\begin{question}[Rate-controlling invariants beyond $\LVC$]
\label{q:rates}
Is there a simple invariant of $\F$, extending or refining
$\LVC(\F)$, that governs the \emph{sample--complexity rate} of uniform estimation
in product spaces?
More concretely, can one characterize when explicit dependence on the width $d$ is necessary, and when it can be avoided?
\end{question}

\vspace{-1em}
\subsubsection{Infinite product spaces}\label{subsubsec:infinite-product}
To probe these issues further, we move to the infinite--product setting.
{Here the width is unbounded, so any phenomenon that forces sample complexity to
grow with $d$ can no longer be hidden in quantitative constants and instead
becomes a qualitative obstruction to uniform estimability.
This makes infinite products a natural test case for identifying what features
of a class $\F$ (beyond $\LVC(\F)$) control estimability.
In particular, characterizing uniform estimability on infinite product spaces
provides an extreme starting point for addressing
Question~\ref{q:rates}.}

We focus on the infinite Boolean cube $\X=\{0,1\}^{\mathbb N}$ equipped with its
canonical product $\sigma$--algebra.
In this setting, we take $\F$ to be the family of all measurable events
(i.e., the full product $\sigma$-algebra on $\X$), and let $\Pc$
denote the family of all product distributions on $\X$.
Example~\ref{ex:d-dependence} already implies that in this setting bounded linear
VC dimension is insufficient to guarantee uniform estimability.
{Thus, in the infinite-dimensional product setting the linear VC dimension loses
its role: while it yields a sharp characterization in finite-dimensional product
spaces, it becomes too weak to control uniform estimability once the product has
infinitely many coordinates.}

This raises a natural question: what, if anything, should replace it?
One possibility is that in infinite product spaces the benefit of product
structure disappears altogether.
For instance, on the infinite Boolean cube, could uniform estimability with
respect to product measures already be as hard as uniform estimability with
respect to arbitrary distributions?
Equivalently, might uniform estimability in the infinite boolean cube be governed simply by
the classical VC dimension?
The following example shows that this is also false.
\refstepcounter{examplectr}
\paragraph{Example~\theexamplectr\ (Infinite VC dimension with trivial uniform estimation).}
\label{ex:tail-trivial}
Let $\X=\{0,1\}^{\mathbb N}$ and let $\Pc_{\mathrm{prod}}$ denote the family of all
product measures on $\X$.
For $\alpha\in[0,1]$, define the event

\vspace{-0.5em}
\[
E_\alpha
:=
\Bigl\{
  x\in\X :
  \lim_{n\to\infty}\tfrac{1}{n}\sum_{i=1}^n x_i = \alpha
\Bigr\},
\]

\vspace{-0.3em}
and let $\F$ be the class generated by finite unions of such events with rational
parameters $\alpha$.
Then $\VC(\F)=\infty$, yet the pair $(\F,\Pc_{\mathrm{prod}})$ is uniformly
estimable.
In fact, uniform estimation is trivial in this case: a single sample suffices,
so $m(\varepsilon,\delta)=1$.
This follows from Kolmogorov’s $0$-$1$ law, which implies that under any product
measure every event in $\F$ has probability either $0$ or $1$, and is therefore
almost surely determined by a single draw. See Appendix~\ref{sec:infinite-product} for details.

\smallskip

These examples show that neither $\LVC(\F)$ nor $\VC(\F)$ alone governs uniform estimability on $\{0,1\}^{\mathbb N}$. We conclude with the following open problem.

\vspace{-0.4em}
\begin{question}[Uniform estimability in infinite product spaces]
\label{q:infinite}
Is there a natural invariant that characterizes uniform estimability in infinite
product spaces?
\end{question}

\vspace{-2em}
\section{Technical Overview}\label{sec:technicaloverview}
We provide a high-level overview of the proof of
Theorem~\ref{thm:characterization}, focusing on the sufficiency direction.
The necessity of finite linear VC dimension follows from standard no-free-lunch
lower bounds for VC classes and is therefore deferred to \autoref{sec:proof-main}.
\vspace{-0.5em}
\paragraph{Upper bound: proof idea.}
Our upper-bound estimator proceeds in two steps (see also the pseudocode below):
\begin{itemize}
\item We first discretize the class $\F$, using a finite grid, thereby
reducing the problem to a finite subclass.
\item We then apply uniform estimation to this finite restriction and extend the
resulting estimates back to the full class $\F$.
\end{itemize}

Fix accuracy and confidence parameters $\eps,\delta\in(0,1)$, and let
$S\sim P^{m}$ be an i.i.d.\ sample from $P$.
Our estimator first splits $S$ into two disjoint subsamples $S=(S^{(0)},S^{(1)})$.
The subsample $S^{(0)}$ is used for the first phase, and $S^{(1)}$ is used
for the second phase.

\smallskip

\noindent\emph{Phase 1: discretization via representatives.}
The key idea of this phase is to obtain a finite discretization of the class $\F$ by constructing a grid $G$, which in turn is used to define a finite
representative subclass of $\F$.
The role of the grid is to ensure that agreement on $G$ forces closeness in
probability under the target distribution $P$.

Using the initial sample $S^{(0)}$, we form a product grid

\vspace{-0.5em}
\[
G = G_1 \times \cdots \times G_d,
\qquad
G_i := \pi_i(S^{(0)}) \subseteq \W_i ,
\]

where $\pi_i$ denotes projection onto the $i$-th coordinate.
That is, $G$ is obtained by projecting the sample onto each coordinate separately
and taking the Cartesian product of the resulting coordinate sets.

The crucial property we establish is a \emph{hitting property} of the grid $G$
with respect to the distribution $P$: with high probability over the draw of
$S^{(0)}$, the grid intersects the symmetric difference of any two sets whose
$P$-measure is large. 
{Equivalently, $G$ is an $\varepsilon$--net for the family of symmetric differences
induced by $\F$ under the pseudo-metric
$d_P(F,F') := P(F\triangle F')$.}
Formally, we show that, simultaneously for all $F,F'\in\F$,

\vspace{-0.5em}
\[
P(F\triangle F') > \eps/2
\quad\Longrightarrow\quad
(F\triangle F')\cap G \neq \emptyset .
\]

\vspace{-0.1em}
The proof proceeds in two steps.
First, we establish the analogous statement with respect to the
product-of-marginals distribution $P_{\square}$.
Second, uniform box-continuity allows us to transfer this guarantee from
$P_{\square}$ to $P$: since
$P(F\triangle F') > \varepsilon/2$ implies
$P_{\square}(F\triangle F') > \beta(\varepsilon/2)$,
it suffices to ensure that $G$ intersects every symmetric difference with
$P_{\square}$-measure at least $\beta(\varepsilon/2)$.
The required bounds are provided by
Lemmas~\ref{lem:hitting-FDeltaF-LVC} and~\ref{lem:hitting-FDeltaF-beta}, proved in
the appendix.

As a consequence, the grid $G$ induces an equivalence relation on $\F$ given by

\vspace{-0.5em}
\[
F \sim_G F'
\quad\Longleftrightarrow\quad
F\cap G = F'\cap G .
\]

\vspace{-0.3em}
Let $\F_G \subseteq \F$ denote a choice of one representative from each
$\sim_G$-equivalence class.
This yields a finite representative family that discretizes $\F$ at resolution
$\varepsilon/2$ with respect to $P$.
\smallskip

\noindent\emph{Phase 2: estimation on the representative family.}
Having constructed the finite representative family $\F_G$, we now turn to the
second phase, in which uniform estimation over $\F$ is reduced to uniform
estimation over $\F_G$.
For each $F\in\F$, let $F^\star\in\F_G$ denote its representative under
$\sim_G$.
On the high-probability event from Phase~1, the hitting property implies that
agreement on $G$ forces proximity under $P$: if $F\sim_G F^\star$, then
$(F\triangle F^\star)\cap G=\emptyset$, and hence
\(
P(F\triangle F^\star)\le \eps/2 .
\)
Consequently, every set $F\in\F$ can be \emph{rounded} to its representative
$F^\star\in\F_G$ at a cost of $\leq\eps/2$ difference in probability.
Using the independent subsample $S^{(1)}$, we estimate $P(F^\star)$ for each
$F^\star\in\F_G$ by the empirical estimator $\widehat P_{\diag}$.
Since $\F_G$ is finite, standard uniform convergence bounds for finite classes
(e.g., Hoeffding's inequality combined with a union bound) imply that, for a
suitable choice of $m_1$, with probability at least $1-\delta/2$,
\(
\sup_{F^\star\in\F_G}
\bigl|\widehat P_{\diag}(F^\star)-P(F^\star)\bigr|
\le \eps/2.
\)
We define the final estimator by extension through representatives,

\vspace{-0.5em}
\[
\widehat P(F):=\widehat P_{\diag}(F^\star),\qquad F\in\F .
\]

\vspace{-0.3em}
Combining the rounding error from Phase~1 with the estimation error above yields,
on the intersection of the two high-probability events,
\(
|P(F)-\widehat P(F)|\le \eps,
\)
for all $F\in\F$.
We summarize the resulting algorithm below.

\vspace{-0.5em}
\begin{center}
\fbox{\begin{minipage}{0.94\linewidth}
\small
\textbf{Product-Grid Estimation Algorithm.}

\textbf{Input:} accuracy $\eps\in(0,1)$, confidence $\delta\in(0,1)$, sample
$S=(Z_1,\ldots,Z_m)\in\X^m$.

\textbf{Phase I: Discretization via representatives.}

\begin{list}{}{\leftmargin=1.6em \labelsep=0pt \itemindent=-1.6em \itemsep=1.2ex}
\item[\textbf{1. Construct a product grid. }]
Split the sample into two disjoint parts $S=(S^{(0)},S^{(1)})$.
For each coordinate $i\in[d]$, let $G_i:=\pi_i(S^{(0)})$ denote the
projection of $S^{(0)}$ onto the
$i$-th coordinate, and define the grid
\[
G := G_1 \times \cdots \times G_d \subseteq \X .
\]
\end{list}

\textbf{Phase II: Discretization, estimation, and extension.}

\begin{list}{}{\leftmargin=1.6em \labelsep=0pt \itemindent=-1.6em \itemsep=1.2ex}
\item[\textbf{2. Discretization via representatives. }]
Partition $\F$ into equivalence classes according to their traces on $G$,
\[
F \sim F' \iff F\cap G = F'\cap G ,
\]
and let $\F_G$ be a set of representatives.

\item[\textbf{3. Estimation on the representative family. }]
For each $F^\star\in\F_G$, estimate $P(F^\star)$ by the empirical mean over
$S^{(1)}$:

\vspace{-0.3em}
\[
\widehat P(F^\star)
=
\frac{1}{|S^{(1)}|}
\sum_{z\in S^{(1)}} \mathbf{1}\{z\in F^\star\}.
\]
\end{list}

{\textbf{Output:} for each $F\in\F$, let $F^\star\in\F_G$ satisfy $F\cap G = F^\star\cap G$, and return
\(
\widehat P(F):=\widehat P(F^\star).\)}
\end{minipage}}
\end{center}

\paragraph{Sample complexity bound. }
The sample complexity bound decomposes according to the two phases of the
procedure.

\noindent \emph{Discretization via representatives. }
The first subsample of size $m_0$ is used to construct a product grid $G$ with
the property that agreement on $G$ implies closeness in probability.
Lemma~\ref{lem:hitting-FDeltaF-beta} shows that, under the uniform
box-continuity assumption, this holds with probability at least $1-\delta/2$
provided that
\[
m_0 \;\gtrsim\;
\frac{d^2}{\beta(\eps/2)^2}
\Bigl(g+\log\tfrac{1}{\delta}\Bigr).
\]

\noindent \emph{Estimation on the representative family. }
The discretization step yields a finite representative family~$\F_G$.
Once $\F_G$ is fixed, estimating the probabilities of all
$F^\star\in\F_G$ reduces to a standard finite-class uniform estimation problem.
Standard uniform convergence bounds for finite classes imply that
$m_1 = O\!\bigl((\log|\F_G|+\log(1/\delta))/\eps^2\bigr)$ samples suffice.
Accordingly, the second-stage sample complexity depends only on
$\log|\F_G|$, the number of distinct traces induced by $\F$ on the grid $G$.

Thus, controlling the overall sample complexity reduces to upper bounding
$|\F_G|$, the number of distinct traces induced by $\F$ on the grid $G$.
{To achieve this, we prove a variant of the Sauer–Shelah–Perles lemma
\citep{sauer1972,shelah1972} in terms of the linear VC dimension, which we discuss next. }


\paragraph{Grid SSP lemma.}
Given a grid $N=A_1\times A_2\ldots\times A_d$, where each $A_i\subseteq\W_i$ we seek to upper bound the number of distinct traces
$\{F\cap N : F\in\F\}$ in terms of the linear VC dimension.

\begin{lemma}[Grid Sauer-Shelah-Perles bound, informal]
\label{lem:grid-ssp-informal}
If $\LVC(\F)=g<\infty$, then for any grid
$N=A_1\times\cdots\times A_d$ one has
\[
\log_2 \bigl|\{F\cap N : F\in\F\}\bigr|
\;\le\;
{\,O\!\left(g\,n^{d-1}{\log (n/g)}\right)},
\]
where $n := \max_i |A_i|$. {Moreover, for every $d,g$ and $n$, there exist set families $\F$ with
$\LVC(\F)=g$ for which

\vspace{-0.3em}
\[
\log_2 \bigl|\{F\cap N : F\in\F\}\bigr|
\;\ge\;
{\,\Omega\!\left(g\,n^{d-1}\log (n/g)\right)} .
\]}
\end{lemma}
In particular, the dependence on $n$ in Lemma~\ref{lem:grid-ssp-informal} is optimal up to constant factors.
As an example, consider the case $g=1$.
Let $N=[n]^d$, and let $\F$ consist of all subsets $F\subseteq N$ that contain at
most one point on every axis-parallel line.
Equivalently, for each choice of a coordinate $i\in[d]$ and every fixing of the
remaining $d-1$ coordinates, $F$ contains at most one point on the corresponding
line.
Such sets are commonly referred to as $(d\!-\!1)$-dimensional permutations.
In the special case $d=2$, this notion reduces to permutation matrices: subsets
of $[n]\times[n]$ that contain one point in each row and each column.

This class satisfies $\LVC(\F)=1$.
By a result of \citet{keevash2018} (Theorem~1.8), the number of
$(d\!-\!1)$-dimensional permutations grows as
\[
\bigl|\F\bigr|
\;=\;
\Bigl(\tfrac{n}{e^{\,d-1}}+o(n)\Bigr)^{n^{d-1}} ,
\]
which matches the upper bound in Lemma~\ref{lem:grid-ssp-informal} up to constant.
Lower bounds for general $g$ are obtained by taking unions of $g$ such
$(d\!-\!1)$-dimensional permutations.
These constructions yield classes with $\LVC(\F)=g$ whose number of distinct
traces on $N$ grows as
$2^{\,\Omega(g\,n^{d-1}\log (n/g))}$, as stated above.



\label{sec:upper-bound}
\bibliography{ref}
\appendix
\section{Proof of Theorem~\ref{thm:characterization}}\label{sec:proof-main}

\paragraph{Overview.}
This section collects the proofs underlying
Theorem~\ref{thm:characterization}.
The implication $(1)\Rightarrow(2)$ is established by
Proposition~\ref{prop:upper}, which shows that finite linear VC dimension
suffices for uniform estimability under any uniformly box--continuous family
of distributions, together with a quantitative sample--size bound proved in
Lemma~\ref{lem:m-bound-prop1}.
The converse implication $(2)\Rightarrow(1)$ follows from
Proposition~\ref{prop:lower}, which constructs a family of product
distributions under which uniform estimation is impossible when
$\LVC(\F)=\infty$.

\begin{restatedtheorem}{thm:characterization}
Let $\X = \W_1 \times \cdots \times \W_d$ be a product measurable space, and let
$\F \subseteq 2^{\X}$ be a family of events.
Then, the following are equivalent:
\begin{enumerate}
\item $\LVC(\F) < \infty$.
\item For every uniformly box--continuous family of distributions
$\Pc \subseteq \Delta(\X)$, the pair $(\F,\Pc)$ is uniformly estimable.
\end{enumerate}
\end{restatedtheorem}

\begin{proof}
We prove the two implications separately.

\medskip
\noindent\textbf{(1)$\Rightarrow$(2).}
Assume that $\LVC(\F)=g<\infty$.
Let $\Pc\subseteq\Delta(\X)$ be an arbitrary uniformly box--continuous family of
distributions.
By definition, there exists a nondecreasing modulus
$\beta:(0,1]\to(0,1]$ such that every $P\in\Pc$ is box--continuous with modulus
$\beta$, i.e., $\Pc\subseteq\Pc_\beta$.

Fix $\eps,\delta\in(0,1)$.
Applying Proposition~\ref{prop:upper} with accuracy $\eps$ and confidence
$\delta$ to the family
$\Pc_\beta$, we obtain a sample size
$m_0 = m(\eps,\delta,g,d,\beta(\eps/2))$
and an estimator $\widehat P:\X^\star\to[0,1]^{\F}$ such that for every
$m\ge m_0$ and every $P\in\Pc_\beta$, with probability at least $1-\delta$
over the sample $S\sim P^{m}$,
\[
\sup_{F\in\F}\bigl|P(F)-\widehat P(S)(F)\bigr|\le \eps.
\]
Since $\Pc\subseteq\Pc_\beta$, the same estimator and sample size satisfy the
same guarantee for every $P\in\Pc$.
As $\eps,\delta\in(0,1)$ were arbitrary, this shows that $(\F,\Pc)$ is uniformly
estimable.

\medskip
\noindent\textbf{(2)$\Rightarrow$(1).}
We prove the contrapositive.
Assume that $\LVC(\F)=\infty$.
Let $\Pc_0$ denote the family of all product distributions on $\X$; this family
is uniformly box--continuous.
By Proposition~\ref{prop:lower}
(Non--estimability when $\LVC(\F)=\infty$),
the pair $(\F,\Pc_0)$ is not uniformly estimable.
Hence statement~(2) fails, since it requires uniform estimability for
\emph{every} uniformly box--continuous family.
This proves the contrapositive, and therefore $(2)\Rightarrow(1)$.

\medskip
Combining the two implications completes the proof.
\end{proof}

\begin{proposition}[Uniform estimability under bounded linear VC dimension]
\label{prop:upper}
Let $\X=\W_1\times\cdots\times\W_d$ be a product measurable space, and let
$\F\subseteq 2^{\X}$ be a family of measurable sets with
$\LVC(\F)=g<\infty$.
Let $\Pc\subseteq\Delta(\X)$ be a family of probability distributions that is
uniformly box--continuous  with modulus
$\beta:(0,1]\to(0,1]$.

Then for every $\eps,\delta\in(0,1)$ there exists a sample size
$m_0 \;=\; m(\eps,\delta,g,d,\beta(\eps/2))$
and an estimator $\widehat P:\X^\star\to[0,1]^{\F}$\footnote{Here
$\X^\star:=\bigcup_{n\in\mathbb N}\X^n$ denotes the set of all finite samples.}
such that for every $m\ge m_0$ and every $P\in\Pc$, with probability at least
$1-\delta$ over the sample $S\sim P^{m}$,
\[
\sup_{F\in\F}\bigl|P(F)-\widehat P(S)(F)\bigr|\;\le\;\eps.
\]
In particular, the pair $(\F,\Pc)$ is uniformly estimable.
\end{proposition}

\begin{proof}[Proof of Proposition~\ref{prop:upper}]
Fix $\eps,\delta\in(0,1)$.
We construct an estimator $\widehat P:\X^\star\to[0,1]^{\F}$ and show that there
exists a sample size
$m_0 = m(\eps,\delta,g,d,\beta(\eps/2))$
such that for every $m\ge m_0$ and every $P\in\Pc$,
\[
\Pr_{S\sim P^{m}}
\Bigl[
\sup_{F\in\F}\bigl|P(F)-\widehat P(S)(F)\bigr|\le \eps
\Bigr]\ge 1-\delta.
\]

\medskip
\noindent\textbf{Choice of the sample sizes.}
We choose integers $m_0$ and $m_1$ (specified below) and set
\[
m := m_0+m_1.
\]
Given a sample $S\sim P^{m}$, we write $S=(S^{(0)},S^{(1)})$, where
\[
S^{(0)}=(X^{(1)},\ldots,X^{(m_0)})\sim P^{m_0},
\qquad
S^{(1)}=(Y^{(1)},\ldots,Y^{(m_1)})\sim P^{m_1}
\]
are disjoint subsamples.

Let $G := G(S^{(0)})$ be the sample grid induced by $S^{(0)}$ (Definition~\ref{def:sample-grid}). Choose $m_0$ large enough so that Lemma~\ref{lem:hitting-FDeltaF-beta} holds with
parameters $(\eps/2,\delta/2)$.
On the corresponding event, which has probability at least $1-\delta/2$, the
following holds simultaneously for all $F,F'\in\F$:
\[
F\cap G = F'\cap G
\quad\Longrightarrow\quad
|P(F)-P(F')|\le \eps/2.
\]

Define an equivalence relation on $\F$ by
$F\sim_G F'$ if and only if $F\cap G=F'\cap G$, and let
$\F_G\subseteq\F$ be a set of representatives containing exactly one element from
each equivalence class.
For each $F\in\F$, denote by $F^\star\in\F_G$ its representative.
Then, on the same event,
\begin{equation}\label{eq:rep-approx-prop}
\forall F\in\F,\qquad |P(F)-P(F^\star)|\le \eps/2.
\end{equation}

Using the second subsample $S^{(1)}$, define the empirical estimator
\[
\widehat P_{\diag}^{(1)}(A)
\;:=\;
\frac{1}{m_1}\sum_{i=1}^{m_1}\mathbf 1\{Y^{(i)}\in A\},
\qquad A\subseteq\X.
\]
We define the final estimator by
\[
\widehat P(S)(F):=\widehat P_{\diag}^{(1)}(F^\star),\qquad F\in\F.
\]

Conditioning on $S^{(0)}$, the variables
$\mathbf 1\{Y^{(i)}\in F^\star\}$ are i.i.d., bounded in $[0,1]$,
with mean $P(F^\star)$ and the family $\F_G$ is finite and fixed.
For any fixed $F^\star\in\F_G$, Hoeffding's inequality gives
\[
\Pr\Bigl(
\bigl|\widehat P_{\diag}^{(1)}(F^\star)-P(F^\star)\bigr|>\eps/2
\ \Big|\ S^{(0)}
\Bigr)
\;\le\;
2\exp\!\left(-\frac{m_1\eps^2}{2}\right).
\]
Applying a union bound over all $F^\star\in\F_G$ yields
\[
\Pr\Bigl(
\sup_{F^\star\in\F_G}
\bigl|\widehat P_{\diag}^{(1)}(F^\star)-P(F^\star)\bigr|>\eps/2
\ \Big|\ S^{(0)}
\Bigr)
\;\le\;
2|\F_G|\exp\!\left(-\frac{m_1\eps^2}{2}\right).
\]
Thus, if
\[
m_1
\;\ge\;
\frac{2}{\eps^2}
\log\!\left(\frac{4|\F_G|}{\delta}\right),
\]
then
\begin{equation}\label{eq:finite-est-prop}
\Pr\Bigl(
\sup_{F^\star\in\F_G}
\bigl|\widehat P_{\diag}^{(1)}(F^\star)-P(F^\star)\bigr|
\le \eps/2
\ \Big|\ S^{(0)}
\Bigr)\ge 1-\delta/2.
\end{equation}

On the intersection of the events in
\eqref{eq:rep-approx-prop} and~\eqref{eq:finite-est-prop},
for every $F\in\F$,
\[
\bigl|P(F)-\widehat P(S)(F)\bigr|
\le |P(F)-P(F^\star)|
+|P(F^\star)-\widehat P_{\diag}^{(1)}(F^\star)|
\le \eps.
\]
By a union bound, this event occurs with probability at least $1-\delta.$
This completes the proof.
\end{proof}
\begin{lemma}[Sample size bound for Proposition~\ref{prop:upper}]
\label{lem:m-bound-prop1}
Under the assumptions of Proposition~\ref{prop:upper}, let $\beta$ be the common
modulus of uniform box--continuity of $\Pc$, and write
$M_0:=C_0\,\frac{d^2}{\beta(\eps/2)^2}\Bigl(g+\log\tfrac{1}{\delta}\Bigr),$
where $C_0>0$ is the universal constant from
Lemma~\ref{lem:hitting-FDeltaF-beta}.
Then there exists a (possibly different) universal constant $C>0$ such that the
estimator constructed in the proof of Proposition~\ref{prop:upper} succeeds
whenever
\[
m
\;\ge\;
C\left[
M_0
\;+\;
\frac{1}{\eps^2}
\left(
g\,M_0^{d-1}\log\!\Bigl(\tfrac{eM_0}{g}\Bigr)
+
\log\tfrac{1}{\delta}
\right)
\right].
\]

In particular,
\[
m
\;=\;
\tilde O\!\left(
\frac{{C_0}^d \cdot d^{2(d-1)}}{\eps^2\,\beta(\eps/2)^{2(d-1)}}
\cdot 
g\,
\Bigl(g+\log\tfrac{1}{\delta}\Bigr)^{d-1}
\right),
\]
where $\tilde O(\cdot)$ hides polylogarithmic factors.
\end{lemma}

\begin{proof}
We keep the notation from the proof of Proposition~\ref{prop:upper}.
Write $m=m_0+m_1$ for the sample split, and let
$G=G(S^{(0)})=\prod_{i=1}^d A_i$ be the empirical grid formed from $S^{(0)}$.

\medskip
\noindent\textbf{Step 1: choice of $m_0$.}
By Lemma~\ref{lem:hitting-FDeltaF-beta} with parameters $(\eps/2,\delta/2)$, it
suffices to take
\[
m_0 \;\ge\; M_0.
\]

\medskip
\noindent\textbf{Step 2: bound on $|\F_G|$.}
Let $G=\prod_{i=1}^d A_i$ be the empirical grid induced by $S^{(0)}$.
Then $|A_i|\le m_0$ for all $i$, and hence $n:=\max_i |A_i|\le m_0$.
Assuming $m_0\ge g$, Corollary~\ref{cor:grid-ssp-rate} gives
\[
\log|\F_G|
=
\log\bigl|\{F\cap G:F\in\F\}\bigr|
\;\le\;
g\,n^{d-1}\log_2\!\Big(\tfrac{en}{g}\Big)
\;\le\;
g\,m_0^{d-1}\log_2\!\Big(\tfrac{e m_0}{g}\Big),
\]
where we absorbed constant factors.

\medskip
\noindent\textbf{Step 3: choice of $m_1$.}
As shown in Step~(ii) of the proof of Proposition~\ref{prop:upper}, conditioning
on $S^{(0)}$ it suffices to choose
\[
m_1
\;\ge\;
\frac{2}{\eps^2}
\left(\log\!\bigl(4|\F_G|\bigr)+\log\tfrac{1}{\delta}\right).
\]
Substituting the bound on $\log|\F_G|$ yields that it suffices to take
\[
m_1
\;\ge\;
C\cdot \frac{1}{\eps^2}\left(
g\,m_0^{d-1}\log\!\Bigl(\tfrac{e m_0}{g}\Bigr)
+\log\tfrac{1}{\delta}\right)
\]
for a universal constant $C>0$.

Taking $m_0=M_0$ and $m=m_0+m_1$ gives the stated bound.
The $\tilde O(\cdot)$ form follows by substituting
$M_0=C_0 d^2\beta(\eps/2)^{-2}\bigl(g+\log(1/\delta)\bigr)$ and absorbing logarithmic
factors.
\end{proof}
\begin{proposition}[Lower bound via infinite linear VC dimension]
\label{prop:lower}
Let $\X=\W_1\times\cdots\times\W_d$ and let $\F\subseteq 2^{\X}$ be a family of
measurable sets with $\LVC(\F)=\infty$.
Let $\Pc_{\mathrm{prod}}$ be the family of all product distributions on $\X$.
Then $(\F,\Pc_{\mathrm{prod}})$ is not uniformly estimable.

More precisely, for every sample--size function $m=m(\eps,\delta)$ there exist
parameters $\eps,\delta\in(0,1)$ and a distribution $P\in\Pc_{\mathrm{prod}}$
such that for any estimator $\widehat P$ based on $m(\eps,\delta)$ i.i.d.\ samples,
\[
\Pr_{S\sim P^{\otimes m(\eps,\delta)}}
\Bigl[
\sup_{F\in\F}\bigl|P(F)-\widehat P(F)\bigr|>\eps
\Bigr]
\;\ge\;\delta.
\]
\end{proposition}

\begin{proof}
Fix an arbitrary sample--size function $m=m(\eps,\delta)$.
Let $\eps:=\eps_0$ and $\delta:=1/4$, where $\eps_0>0$ is the universal constant from a
standard VC lower bound (See \cite{shalev2014understanding, blumer1989}).

Set $m_0:=m(\eps,\delta)$.
Since $\LVC(\F)=\infty$, there exist a coordinate $j\in[d]$ and a section
$x_{-j}\in \prod_{i\neq j}\W_i$ such that the associated section class
\[
\F|_{x_{-j}}
\;:=\;
\bigl\{\,F_{x_{-j}} : F\in\F\,\bigr\}
\;\subseteq\;
2^{\W_j},
\]
where
\[
F_{x_{-j}}
\;:=\;
\bigl\{\,x_j\in\W_j : (x_{-j},x_j)\in F\,\bigr\},
\]
has VC dimension at least $g$, for some $g>m_0$.
In particular, $\F|_{x_{-j}}$ shatters a set
$\{z_1,\ldots,z_g\}\subseteq\W_j$.

Let $Q_j$ be the uniform distribution on $\{z_1,\ldots,z_g\}$, and define the
product distribution $P$ on $\X$ by
\[
P
=
\delta_{x_1}\otimes\cdots\otimes \delta_{x_{j-1}}
\;\otimes\;
Q_j
\;\otimes\;
\delta_{x_{j+1}}\otimes\cdots\otimes \delta_{x_d}.
\]

Under $P$, sampling $S\sim P^{\otimes m_0}$ is equivalent to sampling
$m_0$ i.i.d.\ points from $Q_j$ on $\W_j$, since all other coordinates are fixed.
Moreover, for every $F\in\F$,
\[
P(F)
=
Q_j\!\left(F_{x_{-j}}\right).
\]
Hence estimating the probabilities $\{P(F):F\in\F\}$ is at least as hard as
estimating the probabilities
$\{Q_j(A):A\in \F|_{x_{-j}}\}$ on the one--dimensional domain $\W_j$.

Since $\VC(\F|_{x_{-j}})\ge g>m_0$, the classical VC lower bound implies that for
any estimator $\widehat P$ based on $m_0$ samples,
\[
\Pr_{S\sim P^{\otimes m_0}}
\Bigl[
\sup_{F\in\F}\bigl|P(F)-\widehat P(F)\bigr|>\eps
\Bigr]
\;\ge\;\delta.
\]
This shows that no sample--size function $m(\eps,\delta)$ can guarantee uniform
estimation over $\Pc_{\mathrm{prod}}$, and therefore $(\F,\Pc_{\mathrm{prod}})$ is
not uniformly estimable.
\end{proof}

\section{Additional Lemmas}

\subsection{Uniform control under product measures}

\subsubsection{The empirical product estimator}
\label{subsec:emp-prod}

Let $\X=\W_1\times\cdots\times\W_d$ be a product measurable space.
For a sample
\[
S=(X^{(1)},\ldots,X^{(m)})\sim P^{\otimes m},
\]
define, for each coordinate $j\in[d]$, the empirical marginal distribution
$\tilde P_j\in\Delta(\W_j)$ by
\[
\tilde P_j(A)
\;:=\;
\frac{1}{m}\sum_{t=1}^m \mathbf 1\{X^{(t)}_j\in A\},
\qquad
A\subseteq \W_j \text{ measurable}.
\]
The \emph{empirical product estimator} associated with $S$ is the product
distribution
\begin{equation}
\label{eq:empirical-product}
\tilde P
\;:=\;
\tilde P_1\otimes\cdots\otimes\tilde P_d
\quad\in\Delta(\X).
\end{equation}

\paragraph{Relation to prior work and statistical structure.}
The empirical product estimator~\eqref{eq:empirical-product} coincides with the
estimator used in the product empirical risk minimization (PERM) framework of
\cite{guo2020}.
Closely related procedures also appear in earlier work of \cite{livni2019},
where empirical products of coordinate-wise marginals are employed to decouple
dependencies across dimensions.

From a statistical perspective, the empirical product estimator $\tilde P$
can be viewed as a \emph{V-statistic} of order $d$, obtained by averaging over
all $d$-tuples formed from the sample with replacement.
While most of the classical literature focuses on \emph{U-statistics}, which
average over tuples without replacement, the same analytical machinery applies
to V-statistics with only minor modifications
(see, e.g., \cite{serfling2009, vaart1998}).

\subsubsection{Expected uniform deviation}

\begin{lemma}[Expected deviation of the empirical product estimator]
\label{lem:prod-vc-d}
Let $\X=\W_1\times\cdots\times\W_d$ be a product measurable space, and let
$\F\subseteq 2^\X$ satisfy $\LVC(\F)=g<\infty$.
Let $P=P_1\otimes\cdots\otimes P_d$ be a product distribution on $\X$, and draw
an i.i.d.\ sample $S\sim P^{m}$.

Let $\tilde P$ denote the empirical product estimator associated with $S$.
Then there exists a universal constant $C>0$ such that
\[
\E_{S\sim P^{m}}
\Bigl[
\sup_{F\in\F}\bigl(\tilde P(F)-P(F)\bigr)
\Bigr]
\;\le\;
C\,d\,\sqrt{\frac{g}{m}}.
\]
\end{lemma}

\begin{proof}
We decompose the product estimator $\tilde P$ into $d$ successive substitutions
of empirical marginals.
\medskip
\noindent\textbf{Step 1: decompose by coordinates.}
Define intermediate product measures
\[
\tilde P^{(0)} := P_1 \otimes \cdots \otimes P_d = P,
\]
and for $j=1,\dots,d$,
\[
\tilde P^{(j)} :=
\tilde P_1 \otimes \cdots \otimes \tilde P_j
\otimes P_{j+1} \otimes \cdots \otimes P_d,
\]
so that $\tilde P^{(d)}=\tilde P$.
Then for any $F\in\F$,
\[
\tilde P(F)-P(F)
= \sum_{j=1}^d \bigl(\tilde P^{(j)}(F)-\tilde P^{(j-1)}(F)\bigr).
\]
By the triangle inequality and linearity of expectation,
\[
\E\Bigl[\sup_{F\in\F}(\tilde P(F)-P(F))\Bigr]
\;\le\;
\sum_{j=1}^d
\E\Bigl[\sup_{F\in\F}
\bigl(\tilde P^{(j)}(F)-\tilde P^{(j-1)}(F)\bigr)\Bigr].
\]
It therefore suffices to bound each summand by $C\sqrt{g/m}$.

\medskip
\noindent\textbf{Step 2: isolate the $j$-th coordinate.}
Fix $j\in[d]$ and condition on all samples except the $j$-th coordinate.
Under this conditioning, define
\[
\tilde P_{-j}
:=
\tilde P_1 \otimes \cdots \otimes \tilde P_{j-1}
\otimes P_{j+1} \otimes \cdots \otimes P_d,
\]
so that
\[
\tilde P^{(j-1)} = \tilde P_{-j} \otimes P_j,
\qquad
\tilde P^{(j)}   = \tilde P_{-j} \otimes \tilde P_j.
\]

For $x=(x_1,\dots,x_d)\in\X$, denote the $j$-th section of $F$ by
\[
F_{x_{-j}}
:=
\{z\in\W_j :
(x_1,\dots,x_{j-1},z,x_{j+1},\dots,x_d)\in F\}.
\]
Then
\[
\tilde P^{(j)}(F)
=
\E_{x_{-j}\sim\tilde P_{-j}}\bigl[\tilde P_j(F_{x_{-j}})\bigr],
\qquad
\tilde P^{(j-1)}(F)
=
\E_{x_{-j}\sim\tilde P_{-j}}\bigl[P_j(F_{x_{-j}})\bigr].
\]
Hence,
\begin{align*}
\sup_{F\in\F}\bigl(\tilde P^{(j)}(F)-\tilde P^{(j-1)}(F)\bigr)
&=
\sup_{F\in\F}
\E_{x_{-j}\sim\tilde P_{-j}}
\bigl[\tilde P_j(F_{x_{-j}})-P_j(F_{x_{-j}})\bigr]\\
&\le
\E_{x_{-j}\sim\tilde P_{-j}}
\sup_{F\in\F}
\bigl(\tilde P_j(F_{x_{-j}})-P_j(F_{x_{-j}})\bigr).
\end{align*}

\medskip
\noindent\textbf{Step 3: VC bound on sections.}
By the definition of linear VC dimension, for every $x_{-j}$ the section class
\[
\F|_{x_{-j}}
:= \{F_{x_{-j}} : F\in\F\} \subseteq 2^{\W_j}
\]
has VC dimension at most $g$.
By the standard VC inequality for empirical measures on $\W_j$,
there exists a universal constant $C>0$ such that
\[
\E_{S_j\sim P_j^m}
\Bigl[
\sup_{F\in\F}
\bigl(\tilde P_j(F_{x_{-j}})-P_j(F_{x_{-j}})\bigr)
\Bigr]
\;\le\;
C\sqrt{\frac{g}{m}},
\]
uniformly over all $x_{-j}$ and all realizations of the other coordinates.

\medskip
\noindent\textbf{Step 4: remove conditioning.}
Since the bound does not depend on $x_{-j}$ or the conditioning,
taking expectations over $x_{-j}$ and all remaining samples yields
\[
\E
\Bigl[
\sup_{F\in\F}
\bigl(\tilde P^{(j)}(F)-\tilde P^{(j-1)}(F)\bigr)
\Bigr]
\;\le\;
C\sqrt{\frac{g}{m}}.
\]

\medskip
\noindent\textbf{Step 5: sum over coordinates.}
Summing over $j=1,\dots,d$ gives
\[
\E_{S\sim P^{m}}
\Bigl[
\sup_{F\in\F} \bigl(\tilde P(F)-P(F)\bigr)
\Bigr]
\;\le\;
C\,d\,\sqrt{\frac{g}{m}},
\]
which completes the proof.
\end{proof}

\begin{remark}[From expectation to high probability for product measures]
\label{rem:product-high-prob}
Lemma~\ref{lem:prod-vc-d} gives an \emph{in-expectation} uniform deviation bound
for the empirical product estimator under a product distribution
$P=P_1\otimes\cdots\otimes P_d$.
Combining it with the McDiarmid-type concentration lemma
(Lemma~\ref{lem:mcdiarmid-prod-sup}) yields a high-probability guarantee and,
consequently, a sharper sample-size requirement in the product case
(corresponding to the box--continuity modulus $\beta(\alpha)=\alpha$).

To simplify the presentation, we assume throughout this remark that the class
$\F$ is closed under complementation.
This assumption is without loss of generality, since replacing $\F$ by
$\F\cup\{\,\X\setminus F:\,F\in\F\}$ increases the linear VC dimension by at most
a constant factor.
Under this assumption,
\[
\sup_{F\in\F}\bigl|P(F)-\tilde P(F)\bigr|
=
\sup_{F\in\F}\bigl(P(F)-\tilde P(F)\bigr),
\]
so absolute values can be omitted.

Indeed, let
\[
\Phi(S):=\sup_{F\in\F}\bigl(P(F)-\tilde P(F)\bigr),
\]
where $\tilde P=\tilde P_1\otimes\cdots\otimes\tilde P_d$ is the empirical
product estimator built from $S\sim P^{\otimes m}$.
By Lemma~\ref{lem:prod-vc-d}, there exists a universal constant $C_0>0$ such that
\[
\E[\Phi(S)] \;\le\; C_0\,d\,\sqrt{\frac{g}{m}}.
\]
Moreover, Lemma~\ref{lem:mcdiarmid-prod-sup} implies that for some universal
$c>0$ and all $t>0$,
\[
\Pr\!\left[\Phi(S)-\E[\Phi(S)]\ge t\right]
\;\le\;
\exp\!\left(-c\,\frac{m t^2}{d^2}\right).
\]

Choosing $m$ so that
\[
C_0\,d\,\sqrt{\frac{g}{m}}\le \frac{\eps}{2},
\qquad
\sqrt{\frac{d^2\log(1/\delta)}{c\,m}}\le \frac{\eps}{2},
\]
yields, with probability at least $1-\delta$,
\[
\sup_{F\in\F}|P(F)-\tilde P(F)| \;\le\; \eps.
\]
In particular, there exists a universal constant $C>0$ such that it suffices to
take
\[
m \;\ge\; C\,\frac{d^2}{\eps^2}\Bigl(g+\log\tfrac{1}{\delta}\Bigr).
\]
This improves substantially over the general $\Pc_\beta$ sample-size bound and
captures the favorable behavior of the product estimator when $P$ is a product
measure.
\end{remark}

\subsection{A finite reduction via a random grid}
The goal of this subsection is to show that, with high probability, the random grid
$G(S)$ from Definition~\ref{def:sample-grid} intersects every ``large'' symmetric difference in $\F\triangle\F$.

\begin{definition}[Sample grid]\label{def:sample-grid}
Let $\X=\W_1\times\cdots\times\W_d$ and let $S=(X^{(1)},\ldots,X^{(m)})\in\X^m$ be a finite sample.
For each $j\in[d]$, define the coordinate projection set
\[
\pi_j(S)\;:=\;\{X^{(1)}_j,\ldots,X^{(m)}_j\}\subseteq \W_j,
\]
and define the associated (empirical) product grid
\[
G(S)\;:=\;\pi_1(S)\times\cdots\times\pi_d(S)\subseteq \X.
\]
\end{definition}

\begin{definition}[Symmetric Difference]\label{def:symmetric-diff}
For two sets $A,B\subseteq \X$, their \emph{symmetric difference} is defined by
\[
A\Delta B \;:=\; (A\setminus B)\,\cup\, (B\setminus A)
\;=\;\{\,x\in\X : \mathbf{1}_A(x)\neq \mathbf{1}_B(x)\,\}.
\]
For a family of sets $\F\subseteq 2^{\X}$ we denote
\[
\F\Delta\F \;:=\; \{\,A\Delta B : A,B\in\F\,\},
\]
that is, the family of all pairwise symmetric differences between members of $\F$.
\end{definition}

\begin{theorem}[VC dimension of aggregation
\texorpdfstring{\citep[Proposition~4.6]{alon2023}}{}]
\label{thm:vc-aggregation}
Let $\mathcal{B} \subseteq \{\pm1\}^{\mathcal{X}}$ be a base class, and let
$\mathcal{G}$ be a class of aggregation rules
$g : \{\pm1\}^T \to \{\pm1\}$.
Then
\[
\VC\!\left(
\bigl\{
x \mapsto g(b_1(x),\dots,b_T(x)) :
b_i \in \mathcal{B},\, g \in \mathcal{G}
\bigr\}
\right)
\;\le\;
c_T\bigl(T\cdot \VC(\mathcal{B}) + \VC(\mathcal{G})\bigr),
\]
where
\[
c_T \;=\; \frac{1}{T \cdot \eta},
\]
and $\eta \in (0,1/2)$ is the unique solution to
\[
H(\eta) \;=\; \frac{1}{T+1},
\]
with $H(\cdot)$ denoting the binary entropy function.
\end{theorem}

\begin{corollary}[VC of symmetric differences]\label{cor:vc-symdiff}
For every class $\F\subseteq 2^{\X}$,
\[
\VC(\F\triangle\F)\ \le\ 20\,\VC(\F).
\]
\end{corollary}

\begin{proof}
Identify sets with $\{\pm1\}$ indicators and let $g(a,b)=a\oplus b$ (XOR), so $T=2$.
Apply Theorem~\ref{thm:vc-aggregation} with $\mathcal B=\F$ and
$\mathcal G=\{g\}$ (so $\VC(\mathcal G)=0$) to get
\[
\VC(\F\triangle\F)\ \le\ 2c_2\,\VC(\F),
\qquad c_2=\frac{1}{2\eta},
\]
where $\eta\in(0,1/2)$ satisfies $H(\eta)=1/3$.
Since $H$ is strictly increasing on $(0,1/2)$ and $H(0.05)\approx 0.286<1/3$, we have
$\eta>0.05$, hence $2c_2=1/\eta<20$.
Thus, for concreteness, we fix the explicit bound
$\VC(\F\triangle\F)\le 20\,\VC(\F)$.
\end{proof}

\begin{lemma}[Linear VC dimension of symmetric differences]\label{lem:lvc-symdiff}
Let $\F \subseteq 2^\X$ be a set system on the product space
$\X = \W_1\times\cdots\times\W_d$ with $\LVC(\F)=g<\infty$.
Then for $\Hc:=\F\triangle\F$,
\[
\LVC(\Hc)\ \le\ 20g.
\]
\end{lemma}

\begin{proof}
Fix $j\in[d]$ and $x_{-j}\in\prod_{i\neq j}\W_i$.
Then $\VC(\F|_{x_{-j}})\le g$ by definition of $\LVC(\F)$.
Moreover, for $H=F\triangle F'$ we have
$H_{x_{-j}}=F_{x_{-j}}\triangle F'_{x_{-j}}$, hence
\[
\Hc|_{x_{-j}}
\subseteq
\F|_{x_{-j}}\triangle \F|_{x_{-j}}.
\]
Since
$\Hc|_{x_{-j}}\subseteq \F|_{x_{-j}}\triangle \F|_{x_{-j}}$,
monotonicity of VC dimension and Corollary~\ref{cor:vc-symdiff} give
\[
\VC(\Hc|_{x_{-j}})
\le
20\,\VC(\F|_{x_{-j}})
\le
20g.
\]
Taking the supremum over $j$ and $x_{-j}$ yields $\LVC(\Hc)\le 20g$.
\end{proof}

\noindent
We next combine the bound on the linear VC dimension of the symmetric--difference
class $\Hc=\F\triangle\F$ with the expected deviation bound for the empirical
product estimator.
This allows us to control deviations of the form $P(F)-\tilde P(F')$
uniformly over pairs $F,F'\in\F$ by reducing them to deviations over $\Hc$.

\begin{lemma}[Expected product deviation for $\Hc$]\label{lem:prod-vc-symdiff-expect}
Let $\X=\W_1\times\cdots\times\W_d$ and $\F\subseteq 2^\X$ with $\LVC(\F)=g<\infty$,
and set $\Hc := \F\triangle\F$.
Let $P=P_1\otimes\cdots\otimes P_d$ be a product distribution on $\X$ and
let $\tilde P$ be the empirical product estimator built from $m$ samples
as in Lemma~\ref{lem:prod-vc-d}.
Then there exists a universal constant $C>0$ such that
\[
\E_{S\sim P^{m}}
\Bigl[
\sup_{H\in \Hc} \bigl(P(H) - \tilde P(H)\bigr)
\Bigr]
\;\le\;
C\,d\,\sqrt{\frac{\LVC(\Hc)}{m}}
\;\le\;
C\,d\,\sqrt{\frac{20g}{m}}.
\]
\end{lemma}

\begin{proof}

Apply Lemma~\ref{lem:prod-vc-d} to the class $\Hc$ in place of $\F$, which yields
\[
\E_{S\sim P^{m}}
\Bigl[
\sup_{H\in \Hc} \bigl(P(H) - \tilde P(H)\bigr)
\Bigr]
\;\le\;
C\,d\,\sqrt{\frac{\LVC(\Hc)}{m}}.
\]
Finally, Lemma~\ref{lem:lvc-symdiff} gives $\LVC(\Hc)\le 20\,\LVC(\F)=20g$,
and the result follows.
\end{proof}

\begin{lemma}[McDiarmid for the product supremum]\label{lem:mcdiarmid-prod-sup}
Let $P=P_1\otimes\cdots\otimes P_d$ be a product distribution on $\X$,
and let $\Hc\subseteq 2^\X$ be any class.
For $S=(X^{(1)},\dots,X^{(m)})\sim P^{m}$, let $\tilde P$ be the
empirical product estimator based on $S$, and define
\[
\Phi(S) := \sup_{H\in\Hc}\bigl(P(H)-\tilde P(H)\bigr).
\]
Then for all $t>0$,
\[
\Pr\bigl[\Phi(S)-\E[\Phi(S)] \ge t \bigr]
\;\le\;
\exp\!\Bigl(-\,\frac{2 m t^2}{d^2}\Bigr).
\]
\end{lemma}

\begin{proof}
Let $S$ and $S'$ be two samples that differ only in the $k$-th point, and let
$\tilde P=\tilde P_1\otimes\cdots\otimes\tilde P_d$ and
$\tilde P'=\tilde P'_1\otimes\cdots\otimes\tilde P'_d$
be the corresponding product estimators.
Changing a single observation affects each empirical marginal by at most $1/m$
in total variation:
\[
\|\tilde P_j-\tilde P'_j\|_{\mathrm{TV}} \le \frac{1}{m},
\qquad j=1,\dots,d.
\]
Using the standard product bound
\[
\bigl\|\textstyle\bigotimes_{j=1}^d \mu_j-\bigotimes_{j=1}^d \nu_j\bigr\|_{\mathrm{TV}}
\le \sum_{j=1}^d \|\mu_j-\nu_j\|_{\mathrm{TV}},
\]
we obtain
\[
\|\tilde P-\tilde P'\|_{\mathrm{TV}}
\le \sum_{j=1}^d \|\tilde P_j-\tilde P'_j\|_{\mathrm{TV}}
\le \frac{d}{m}.
\]
Therefore, for every measurable $H\subseteq\X$,
\[
|\tilde P(H)-\tilde P'(H)| \le \|\tilde P-\tilde P'\|_{\mathrm{TV}} \le \frac{d}{m}.
\]
Since $P$ is fixed, it follows that
\[
|\Phi(S)-\Phi(S')|
=
\left|
\sup_{H\in\Hc}(P(H)-\tilde P(H))
-\sup_{H\in\Hc}(P(H)-\tilde P'(H))
\right|
\le
\sup_{H\in\Hc}|\tilde P(H)-\tilde P'(H)|
\le \frac{d}{m}.
\]
Thus $\Phi$ has bounded differences with constants $c_k=d/m$ for $k=1,\dots,m$.
McDiarmid's inequality yields, for all $t>0$,
\[
\Pr\bigl[\Phi(S)-\E\Phi(S)\ge t\bigr]
\le
\exp\!\left(
-\frac{2t^2}{\sum_{k=1}^m (d/m)^2}
\right)
=
\exp\!\Bigl(-\,\frac{2 m t^2}{d^2}\Bigr),
\]
as claimed.
\end{proof}

\begin{corollary}[High-probability bound for $\Hc=\F\triangle\F$]
\label{cor:prod-vc-symdiff-highprob-self}
There exists a universal constant $C>0$ such that the following holds.
Let $\X=\W_1\times\cdots\times\W_d$ and let $\F\subseteq 2^\X$ satisfy
$\LVC(\F)=g<\infty$, and set $\Hc:=\F\triangle\F$.
Then for every product distribution $P$ on $\X$ and every $\eps,\delta\in(0,1)$,
if
\[
m \;\ge\; C\,\frac{d^2}{\eps^2}\Bigl(g + \log\tfrac{1}{\delta}\Bigr),
\]
then with probability at least $1-\delta$ over $S\sim P^{m}$,
\[
\sup_{H\in\Hc}\bigl(P(H)-\tilde P(H)\bigr)\;\le\;\eps.
\]
\end{corollary}

\begin{proof}
Define
\[
\Phi(S):=\sup_{H\in\Hc}\bigl(P(H)-\tilde P(H)\bigr).
\]
By Lemma~\ref{lem:prod-vc-symdiff-expect} and Lemma~\ref{lem:lvc-symdiff},
\[
\E[\Phi(S)]
\;\le\;
C_0\,d\,\sqrt{\frac{\LVC(\Hc)}{m}}
\;\le\;
C_0\,d\,\sqrt{\frac{20g}{m}},
\]
for a universal constant $C_0>0$.
Moreover, by Lemma~\ref{lem:mcdiarmid-prod-sup},
for all $t>0$,
\[
\Pr\bigl[\Phi(S)\ge \E[\Phi(S)] + t\bigr]
\;\le\;
\exp\!\Bigl(-\,\frac{2mt^2}{d^2}\Bigr).
\]
Taking $t=\eps/2$, we get
\[
\Pr\bigl[\Phi(S)>\eps\bigr]
\;\le\;
\Pr\bigl[\Phi(S)>\E[\Phi(S)]+\tfrac{\eps}{2}\bigr]
\;\le\;
\exp\!\Bigl(-\,\frac{m\eps^2}{2d^2}\Bigr),
\]
provided that $\E[\Phi(S)]\le \eps/2$.
Thus it suffices that
\[
C_0\,d\,\sqrt{\frac{20g}{m}}\le \frac{\eps}{2}
\qquad\text{and}\qquad
\exp\!\Bigl(-\,\frac{m\eps^2}{2d^2}\Bigr)\le \delta,
\]
i.e.,
\[
m \;\ge\; 80\,C_0^2\,\frac{d^2 g}{\eps^2}
\qquad\text{and}\qquad
m \;\ge\; 2\,\frac{d^2}{\eps^2}\log\tfrac{1}{\delta}.
\]
Both are implied by
\[
m \;\ge\; C\,\frac{d^2}{\eps^2}\Bigl(g+\log\tfrac{1}{\delta}\Bigr)
\]
for a sufficiently large universal constant $C>0$, completing the proof.
\end{proof}
\begin{lemma}[Hitting large symmetric differences under product measures]
\label{lem:hitting-FDeltaF-LVC}
Let $\X=\W_1\times\cdots\times\W_d$ be a product space, and let
$\F\subseteq 2^\X$ satisfy $\LVC(\F)=g<\infty$.
Let $\Hc:=\F\triangle\F$.

There exists a universal constant $C>0$ such that the following holds.
For every $\eps,\delta\in(0,1)$, every product distribution
$P=P_1\otimes\cdots\otimes P_d$ on $\X$, and every i.i.d.\ sample
$S\sim P^{m}$, if
\[
m \;\ge\; C\,\frac{d^2}{\eps^2}\Bigl(g+\log\tfrac{1}{\delta}\Bigr),
\]
then with probability at least $1-\delta$ over $S$,
\[
\forall F,F'\in\F \text{ with } P(F\triangle F')\ge\eps:
\qquad
(F\triangle F')\cap G(S)\neq\emptyset,
\]
where $G(S)$ denotes the sample product grid induced by $S$.
Equivalently, with probability at least $1-\delta$, the grid $G(S)$ is an
$\eps$--net for the class $\Hc$ under $P$.
\end{lemma}

\begin{proof}
Let $\Hc := \F\triangle\F$.
By Lemma~\ref{lem:lvc-symdiff}, the symmetric--difference class satisfies
\[
\LVC(\Hc)\;\le\;20\,\LVC(\F)\;=\;20g.
\]

Applying Corollary~\ref{cor:prod-vc-symdiff-highprob-self} with parameter
$\eps/2$ and confidence $\delta$, we obtain that for a universal constant
$C>0$, if
\[
m \;\ge\; C\,\frac{d^2}{\eps^2}\Bigl(g + \log\tfrac{1}{\delta}\Bigr),
\]
then with probability at least $1-\delta$,
\begin{equation}\label{eq:hitting-sup}
\sup_{H\in\Hc}\bigl(P(H)-\tilde P(H)\bigr)
\;\le\;
\frac{\eps}{2}.
\end{equation}

On this event, fix any $F,F'\in\F$ with
\[
P(F\triangle F') \;\ge\; \eps,
\]
and set $H := F\triangle F' \in \Hc$.
Then~\eqref{eq:hitting-sup} implies
\[
\tilde P(H)
\;\ge\;
P(H) - \sup_{K\in\Hc}\bigl(P(K)-\tilde P(K)\bigr)
\;\ge\;
\eps - \frac{\eps}{2}
\;=\;
\frac{\eps}{2}.
\]

Since $\tilde P$ is supported on the grid $G(S)$, the inequality
$\tilde P(H)>0$ implies that $H$ contains at least one point of $G(S)$, i.e.
\[
H\cap G(S) \neq \emptyset.
\]
Equivalently,
\[
(F\triangle F')\cap G(S) \neq \emptyset.
\]

Thus, with probability at least $1-\delta$, every symmetric difference
$F\triangle F'$ of $P$--measure at least $\eps$ intersects the grid $G(S)$,
which completes the proof.
\end{proof}

\begin{lemma}[Hitting large symmetric differences for $P\in\Pc_\beta$]
\label{lem:hitting-FDeltaF-beta}
Let $\X=\W_1\times\cdots\times\W_d$ be a product space, and let
$\F\subseteq 2^\X$ satisfy $\LVC(\F)=g<\infty$.  Set $\Hc:=\F\triangle\F$.

There exists a universal constant $C_0>0$ such that the following holds.
For every $\eps,\delta\in(0,1)$, every $P\in\Pc_\beta$, and every i.i.d.\ sample
$S\sim P^{m}$, if
\[
m \;\ge\; C_0\,\frac{d^2}{\beta(\eps)^2}\Bigl(g+\log\tfrac{1}{\delta}\Bigr),
\]
then with probability at least $1-\delta$ over $S$,
\[
\forall H\in\Hc \text{ with } P(H)\ge \eps:\qquad H\cap G(S)\neq\emptyset,
\]
where $G(S)$ denotes the empirical product grid induced by $S$.
Equivalently, with probability at least $1-\delta$, for all $F,F'\in\F$,
\[
P(F\triangle F')\ge \eps \quad\Longrightarrow\quad (F\triangle F')\cap G(S)\neq\emptyset.
\]
\end{lemma}

\begin{proof}
Let $P_\square$ be the product of marginals of $P$ and set $\eps'=\beta(\eps)$.
By the definition of $\Pc_\beta$, for every $H\in\Hc$,
\[
P(H)\ge\eps \quad\Longrightarrow\quad P_\square(H)\ge\eps'.
\]

Consider the event
\[
\mathcal{E}
:=
\Bigl\{
\forall H\in\Hc \text{ with } P_\square(H)\ge\eps' :\ H\cap G(S)\neq\emptyset
\Bigr\}.
\]
The event $\mathcal{E}$ depends on $S$ only through the coordinate samples
$\{X^{(i)}_j\}_{i\le m}$ (equivalently, through the empirical marginals
$\tilde P_1,\dots,\tilde P_d$). Hence $\Pr(\mathcal{E})$ is the same whether
$S$ is drawn from $P^{m}$ or from $P_\square^{m}$, since under
both distributions each coordinate sample $(X_j^{(1)},\dots,X_j^{(m)})$ is i.i.d.\
from $P_j$.

Therefore, by Lemma~\ref{lem:hitting-FDeltaF-LVC} applied to the product
distribution $P_\square$ with parameters $(\eps',\delta)$, we have
\[
\Pr_{S\sim P^{m}}(\mathcal{E})
=
\Pr_{S'\sim P_\square^{m}}(\mathcal{E})
\;\ge\;
1-\delta,
\]
provided
\[
m \;\ge\; C_0\,\frac{d^2}{(\eps')^2}\Bigl(g + \log\tfrac{1}{\delta}\Bigr).
\]

Finally, on $\mathcal{E}$, if $F,F'\in\F$ satisfy $P(F\triangle F')\ge\eps$ and
we set $H:=F\triangle F'\in\Hc$, then $P_\square(H)\ge\eps'$ and thus
$H\cap G(S)\neq\emptyset$, as required.
\end{proof}

\subsection{Sauer–Shelah–Perles on Grids via Linear VC}\label{sec:proof-grid-ssp}

For convenience, we use the shorthand
\[
\binomle{n}{g} := \sum_{j=0}^{g} \binom{n}{j},
\qquad
\binomle{n}{\infty}:=2^n.
\]

\begin{lemma}[Grid Sauer--Shelah--Perles bound from Linear VC]
\label{lem:grid-ssp}
Let \(\F \subseteq 2^{\W_1\times\cdots\times\W_d}\) and 
\(N = A_1\times\cdots\times A_d \subseteq \X\) be a finite grid with side lengths
\(n_i := |A_i|\).
Fix an index \(i \in [d]\).

If \(\LVC(\F) \le g < \infty\), then
\begin{equation}
\label{eq:grid-ssp}
\big|\{F \cap N : F \in \F\}\big|
\;\le\;
\Bigg(\binomle{n_i}{g}\Bigg)^{\,\prod_{j\ne i} n_j}.
\end{equation}

In particular, letting \(n:=\max_{k\in[d]} n_k\) and choosing \(i\) such that \(n_i=n\),
\begin{equation}
\label{eq:grid-ssp-maxside}
\big|\{F \cap N : F \in \F\}\big|
\;\le\;
\Big(\binomle{n}{g}\Big)^{\,|N|/n}.
\end{equation}
\end{lemma}

\begin{proof}[Proof of Lemma~\ref{lem:grid-ssp}]
Fix \(i\in[d]\).
For each choice of fixed coordinates
\(\mathbf{a}_{-i}\in \prod_{j\ne i}A_j\),
define the axis--parallel line in direction \(i\)
\[
L(\mathbf{a}_{-i})
\;:=\;
\{(x_1,\ldots,x_d)\in N:\ x_j=(\mathbf{a}_{-i})_j \text{ for all } j\ne i,\ \ x_i\in A_i\}.
\]
The family \(\{L(\mathbf{a}_{-i}) : \mathbf{a}_{-i}\in \prod_{j\ne i}A_j\}\)
partitions \(N\) into exactly \(\prod_{j\ne i} n_j\) disjoint lines, each of size \(n_i\).

Since \(\LVC(\F)\le g\), for every \(\mathbf{a}_{-i}\) the restriction
\(\F|_{L(\mathbf{a}_{-i})}\) has VC dimension at most \(g\).
Therefore, by the classical Sauer--Shelah--Perles lemma on a ground set of size \(n_i\),
\[
\big|\{F\cap L(\mathbf{a}_{-i}) : F\in\F\}\big|
\;\le\;
\sum_{j=0}^{g}\binom{n_i}{j}
\;=\;
\binomle{n_i}{g}.
\]

For each \(\mathbf{a}_{-i}\in\prod_{j\ne i}A_j\), the restriction of a set
\(F\in\F\) to the line \(L(\mathbf{a}_{-i})\) is one of at most
\(\binomle{n_i}{g}\) possible traces.
Since the lines \(L(\mathbf{a}_{-i})\) are pairwise disjoint and together
partition \(N\), the collection of line--wise traces
\(\{F\cap L(\mathbf{a}_{-i})\}_{\mathbf{a}_{-i}}\) uniquely determines
the global trace \(F\cap N\).
It follows that the total number of distinct traces on \(N\) is at most the
product of the numbers of possible traces on each line, namely
\[
\big|\{F\cap N : F\in\F\}\big|
\;\le\;
\prod_{\mathbf{a}_{-i}\in \prod_{j\ne i}A_j}
\big|\{F\cap L(\mathbf{a}_{-i}) : F\in\F\}\big|
\;\le\;
\bigg(\binomle{n_i}{g}\bigg)^{\prod_{j\ne i}n_j},
\]
which proves \eqref{eq:grid-ssp}.

To obtain \eqref{eq:grid-ssp-maxside}, let \(n:=\max_{k\in[d]} n_k\) and choose
\(i\in[d]\) such that \(n_i=n\).
Then \(\prod_{j\ne i}n_j = |N|/n\), and substituting this into
\eqref{eq:grid-ssp} yields
\[
\big|\{F \cap N : F \in \F\}\big|
\;\le\;
\Big(\binomle{n}{g}\Big)^{\,|N|/n},
\]
as claimed.
\end{proof}

\begin{corollary}[Grid Sauer--Shelah--Perles bound (rate form)]
\label{cor:grid-ssp-rate}
With the notation of Lemma~\ref{lem:grid-ssp}, let \(n:=\max_i |A_i|\).
If \(1\le g\le n\), then
\begin{equation}
\label{eq:grid-ssp-rate}
\log_2\big|\{F\cap N:F\in\F\}\big|
\;\le\;
\frac{|N|}{n}\,g\,\log_2\!\Big(\tfrac{en}{g}\Big)
\;\le\;
g\,n^{d-1}\log_2\!\Big(\tfrac{en}{g}\Big)
\;=\;
O\!\bigl(g\,n^{d-1}\log(n/g)\bigr).
\end{equation}
Equivalently,
\[
\big|\{F\cap N:F\in\F\}\big|
\;\le\;
2^{\,O\!\left(g\,n^{d-1}\log(n/g)\right)}.
\]
\end{corollary}

\begin{proof}[Proof of Corollary~\ref{cor:grid-ssp-rate}]
Choose \(i\) with \(n_i=n\). By \eqref{eq:grid-ssp} and Lemma~\ref{lem:binomle-bound},
\[
\big|\{F \cap N : F \in \F\}\big|
\;\le\;
\Big(\binomle{n}{g}\Big)^{|N|/n}
\;\le\;
\Big(\tfrac{en}{g}\Big)^{g|N|/n}.
\]
Taking \(\log_2\) yields the first inequality in \eqref{eq:grid-ssp-rate}.
Since \(|N|/n=\prod_{j\ne i}n_j\le n^{d-1}\), the second inequality follows.
\end{proof}


\begin{lemma}
\label{lem:binomle-bound}
For all integers \(n\ge 1\) and \(1\le g\le n\),
\[
\binomle{n}{g}
\;=\;
\sum_{j=0}^{g}\binom{n}{j}
\;\le\;
\Big(\tfrac{en}{g}\Big)^{g}.
\]
Moreover, if \(g\ge n\) then \(\binomle{n}{g}=2^n\).
\end{lemma}

\begin{proof}
Fix integers \(n\ge 1\) and \(1\le g\le n\).
For every \(0\le j\le g\),
\[
\binom{n}{j}
=
\frac{n(n-1)\cdots(n-j+1)}{j!}
\;\le\;
\frac{n^j}{j!},
\]
and therefore
\[
\sum_{j=0}^{g}\binom{n}{j}
\;\le\;
\sum_{j=0}^{g}\frac{n^j}{j!}.
\]
Rewrite each term as
\[
\frac{n^j}{j!}
=
\Big(\tfrac{n}{g}\Big)^{j}\cdot \frac{g^{j}}{j!},
\]
so
\[
\sum_{j=0}^{g}\frac{n^j}{j!}
=
\sum_{j=0}^{g}\Big(\tfrac{n}{g}\Big)^{j}\frac{g^{j}}{j!}.
\]
Since \(1\le g\le n\), we have \(\tfrac{n}{g}\ge 1\), hence
\(\big(\tfrac{n}{g}\big)^j \le \big(\tfrac{n}{g}\big)^g\) for all \(0\le j\le g\).
Thus,
\[
\sum_{j=0}^{g}\Big(\tfrac{n}{g}\Big)^{j}\frac{g^{j}}{j!}
\;\le\;
\Big(\tfrac{n}{g}\Big)^{g}\sum_{j=0}^{g}\frac{g^{j}}{j!}
\;\le\;
\Big(\tfrac{n}{g}\Big)^{g}\sum_{j=0}^{\infty}\frac{g^{j}}{j!}
=
\Big(\tfrac{n}{g}\Big)^{g} e^{g}
=
\Big(\tfrac{en}{g}\Big)^{g}.
\]
Finally, if \(g\ge n\), then \(\sum_{j=0}^{g}\binom{n}{j}=\sum_{j=0}^{n}\binom{n}{j}=2^n\).
\end{proof}

\subsubsection{A matching lower bound for the grid SSP lemma}
\label{sec:lower-bound-ssp}

We show that the upper bound in Lemma~\ref{lem:grid-ssp-informal} is essentially
tight, up to constant factors, for every fixed dimension $d$ and linear VC
dimension $g$.
In particular, the dependence on $n^{d-1}\log n$ in the exponent is unavoidable.

Throughout this subsection, let $N=[n]^d$.
Recall that a \emph{$(d-1)$--dimensional permutation} is a subset
$F\subseteq [n]^d$ that contains exactly one point on every axis--parallel line.
Equivalently, for each choice of a coordinate $i\in[d]$ and every fixing of the
remaining $d-1$ coordinates, $F$ contains exactly one point on the corresponding
line.
Let $\F$ denote the family of all $(d-1)$--dimensional permutations of $[n]^d$.

By a result of \citet{keevash2018} (Theorem~1.8), the cardinality of $\F$ satisfies
\begin{equation}
\label{eq:keevash}
|\F|
=
\Bigl(\tfrac{n}{e^{\,d-1}}+o(n)\Bigr)^{n^{d-1}}.
\end{equation}

\begin{lemma}[A lower bound for the grid SSP rate]
\label{lem:grid-ssp-lower}
Fix $d\ge 2$.
Let $N=[n]^d$ and let $\F$ denote the family of all $(d\!-\!1)$--dimensional
permutations of $[n]^d$, i.e.\ subsets $F\subseteq [n]^d$ that intersect every
axis--parallel line in exactly one point.
Fix an integer $g\ge 1$ and define
\[
\G
\;:=\;
\Bigl\{\; \bigcup_{i=1}^r F_i \ :\ 0\le r\le g,\ \ F_1,\ldots,F_r\in\F \Bigr\}.
\]
Then $\LVC(\G)=g$.
Moreover, 
\begin{equation}
\label{eq:lower-bound-logG}
\log_2|\G|
\;\ge\;
g\,n^{d-1}\Bigl(\log_2\!\tfrac{n}{g\,e^{\,d-1}}+o(1)\Bigr).
\end{equation}
In particular, for fixed $d$ and any $g=o(n)$,
\[
\log_2|\G|
\;=\;
\Omega\!\bigl(g\,n^{d-1}\log(n/g)\bigr).
\]
\end{lemma}

\begin{proof}
\paragraph{$\LVC(\G)=g$.}
Every $F\in\F$ intersects each axis--parallel line in exactly one point.
Hence any union of $r\le g$ members of $\F$ intersects every such line in at
most $r\le g$ points, so $\LVC(\G)\le g$.

For the reverse inequality, fix any axis--parallel line $L\subseteq [n]^d$ and
choose $g$ distinct points $x^{(1)},\ldots,x^{(g)}\in L$.
We claim that $\G|_{L}$ shatters $\{x^{(1)},\ldots,x^{(g)}\}$.
Indeed, for each $j\in[g]$ choose a set $F^{(j)}\in\F$ whose unique point on $L$
is $x^{(j)}$ (this is possible since $\F$ is invariant under independent
permutations of the coordinates, hence contains a permutation through any
prescribed point on $L$).
Then for any subset $S\subseteq \{x^{(1)},\ldots,x^{(g)}\}$, letting
$U_S:=\bigcup_{x^{(j)}\in S} F^{(j)}$ gives $U_S\in\G$ and
$U_S\cap L = S$ (each $F^{(j)}$ contributes exactly one point on $L$).
Thus $\VC(\G|_{L})\ge g$, and taking the supremum over $L$ yields $\LVC(\G)\ge g$.

\paragraph{A counting lower bound on $|\G|$.}
There are $|\F|^g$ ordered $g$--tuples $(F_1,\ldots,F_g)\in\F^g$, each producing
a union $U=\bigcup_{i=1}^g F_i\in\G$.
To lower bound $|\G|$, we upper bound how many $g$--tuples can generate the same
union $U$.

Fix $U\in\G$.
Partition $[n]^d$ into the $n^{d-1}$ axis--parallel lines in (say) direction~$1$.
Since $U$ is a union of $g$ members of $\F$, and each member of $\F$ contributes
exactly one point to each such line, the set $U$ contains at most $g$ points on
every direction~$1$ line.
Therefore, the number of possibilities for a set $F\in\F$ with $F\subseteq U$ is
at most $g^{\,n^{d-1}}$ (on each of the $n^{d-1}$ lines we have at most $g$ choices
for the unique point of $F$).
Consequently, $U$ can be represented as a union of an ordered $g$--tuple from
$\F^g$ in at most $(g^{n^{d-1}})^g = g^{g n^{d-1}}$ ways.
It follows that
\begin{equation}
\label{eq:G-count}
|\G|
\;\ge\;
\frac{|\F|^g}{g^{g n^{d-1}}}.
\end{equation}

\paragraph{Substituting Keevash and taking logs.}
Taking $\log_2$ in \eqref{eq:G-count} gives
\[
\log_2|\G|
\;\ge\;
g\log_2|\F| \;-\; g n^{d-1}\log_2 g.
\]
Using \eqref{eq:keevash},
\[
\log_2|\F|
=
n^{d-1}\log_2\!\Bigl(\tfrac{n}{e^{d-1}}+o(n)\Bigr)
=
n^{d-1}\Bigl(\log_2 n-(d-1)\log_2 e+o(1)\Bigr).
\]
Substituting this and regrouping terms yields
\[
\log_2|\G|
\;\ge\;
g n^{d-1}\Bigl(\log_2 n-\log_2 g-(d-1)\log_2 e+o(1)\Bigr)
=
g n^{d-1}\Bigl(\log_2\tfrac{n}{g e^{d-1}}+o(1)\Bigr),
\]
which is \eqref{eq:lower-bound-logG}.
The final $\Omega(g n^{d-1}\log(n/g))$ statement follows whenever $g=o(n)$.
\end{proof}

\section{Examples}
\subsection{Mixtures of product distributions}\label{sec:proof-mixture}

\begin{restatedproposition}{prop:mixture-box}
Let $P$ be a probability distribution on
\(
\X=\W_1\times\cdots\times\W_d
\)
that can be written as a mixture of at most $k$ product distributions,
\[
P=\sum_{t=1}^k \lambda_t\,
\bigl(\mu^{(t)}_1\otimes\cdots\otimes\mu^{(t)}_d\bigr),
\qquad
\lambda_t\ge 0,\ \sum_{t=1}^k\lambda_t=1.
\]
Then $P$ is uniformly box--continuous with modulus
\[
\beta(\alpha)
\;=\;
\frac{\alpha^d}{(k-1+\alpha)^{\,d-1}}.
\]
\end{restatedproposition}

\begin{proof}
Fix a measurable set $E\subseteq\X$ and define
\[
\phi_t(E)
\;:=\;
(\mu^{(t)}_1\otimes\cdots\otimes\mu^{(t)}_d)(E)
\in[0,1],
\qquad
P(E)=\sum_{t=1}^k \lambda_t\phi_t(E).
\]

Let $P_i=\sum_{t=1}^k\lambda_t\mu^{(t)}_i$ denote the $i$-th marginal of $P$, and
let
\(
P_{\square}:=P_1\otimes\cdots\otimes P_d
\)
be the product of marginals.
Expanding the product of marginals yields
\[
P_{\square}(E)
=
\sum_{t_1,\ldots,t_d=1}^k
\Bigl(\prod_{j=1}^d \lambda_{t_j}\Bigr)
(\mu^{(t_1)}_1\otimes\cdots\otimes\mu^{(t_d)}_d)(E).
\]
All terms are nonnegative, so keeping only the diagonal terms
$t_1=\cdots=t_d=t$ gives
\begin{equation}
\label{eq:mixture-diagonal}
P_{\square}(E)
\;\ge\;
\sum_{t=1}^k \lambda_t^{\,d}\,\phi_t(E).
\end{equation}

Applying Hölder’s inequality with conjugate exponents $d$ and $d/(d-1)$ yields
\[
P(E)
=
\sum_{t=1}^k \lambda_t\phi_t(E)
\;\le\;
\Bigl(\sum_{t=1}^k \lambda_t^{\,d}\phi_t(E)\Bigr)^{1/d}
\Bigl(\sum_{t=1}^k \phi_t(E)\Bigr)^{(d-1)/d}.
\]
Rearranging gives
\begin{equation}
\label{eq:holder-mixture}
\sum_{t=1}^k \lambda_t^{\,d}\phi_t(E)
\;\ge\;
\frac{P(E)^d}{\bigl(\sum_{t=1}^k \phi_t(E)\bigr)^{d-1}}.
\end{equation}

Since $\sum_t\lambda_t=1$ and $0\le\phi_t(E)\le 1$,
\begin{equation}
\label{eq:holder-mixture2}
\sum_{t=1}^k \phi_t(E)
=
\sum_{t=1}^k \lambda_t\phi_t(E)
+
\sum_{t=1}^k (1-\lambda_t)\phi_t(E)
\;\le\;
P(E)+(k-1).
\end{equation}

Combining \eqref{eq:mixture-diagonal}, \eqref{eq:holder-mixture}, and~\eqref{eq:holder-mixture2} we obtain
\[
P_{\square}(E)
\;\ge\;
\frac{P(E)^d}{(k-1+P(E))^{\,d-1}}.
\]
Since the function
\(
x\mapsto x^d/(k-1+x)^{d-1}
\)
is increasing on $(0,\infty)$, replacing $P(E)$ by $\alpha$ completes the proof.
\end{proof}

\begin{lemma}[Lower bound for mixtures]
\label{lem:mixture-unavoidable-symmetric}
Fix integers $d\ge 2$ and $k\ge 2$.
For every $\alpha\in(0,1]$ there exist a finite product space
$\X=\W_1\times\cdots\times\W_d$, a distribution $P$ on $\X$ that is a mixture of
$k$ product distributions, and a measurable set $E\subseteq\X$ such that
\[
P(E)=\alpha
\qquad\text{and}\qquad
P_{\square}(E)=\frac{\alpha^d}{(k-1)^{\,d-1}}.
\]
\end{lemma}

\begin{proof}
Let $\W_1=\cdots=\W_d=[k]$ and $\X=[k]^d$.
For each $t\in[k]$, let $P^{(t)}$ be the product distribution supported on the
single point $(t,\ldots,t)$.

Define mixture weights by
\[
\lambda_t:=\frac{\alpha}{k-1}\quad\text{for }t=1,\ldots,k-1,
\qquad
\lambda_k:=1-\alpha,
\]
and set $P:=\sum_{t=1}^k \lambda_t P^{(t)}$.
Let
\[
E:=\{(t,\ldots,t):\ t=1,\ldots,k-1\}.
\]
Then, clearly $P(E)=\alpha$.

Each marginal $P_i$ equals the distribution on $[k]$ given by $P_i(t)=\lambda_t$,
and hence
\[
P_{\square}(E)
=
\sum_{t=1}^{k-1} \prod_{i=1}^d P_i(t)
=
\sum_{t=1}^{k-1} \lambda_t^d
=
(k-1)\left(\frac{\alpha}{k-1}\right)^d
=
\frac{\alpha^d}{(k-1)^{\,d-1}}.
\]
\end{proof}
\subsection{Bounded total correlation}\label{sec:proof-bounded-info}

\begin{definition}[Total correlation]
\label{def:total-correlation}
Let $\X=\W_1\times\cdots\times \W_d$ be a product measurable space, and let
$P\in\Delta(\X)$ be a probability distribution.
For each $i\in[d]$, let $P_i$ denote the marginal distribution of $P$ on $\W_i$,
and define the \emph{product of marginals} by
\[
P_{\square}
\;:=\;
P_1\otimes\cdots\otimes P_d .
\]
The \emph{total correlation} (also known as \emph{multi--information}) of $P$ is
defined as
\[
\mathrm{TC}(P)
\;:=\;
\KL\!\left(P\,\middle\|\,P_{\square}\right),
\]
where $\KL(\cdot\|\cdot)$ denotes the Kullback--Leibler divergence.
For $d=2$, the total correlation coincides with the mutual information.
\end{definition}

\begin{restatedproposition}{prop:tc-box}
Fix $C\ge 0$ and let
\[
\Pc_C \;:=\; \{\,P\in\Delta(\X):\ \mathrm{TC}(P)\le C\,\}.
\]
Then the family $\Pc_C$ is uniformly box--continuous with modulus
\[
\beta(\alpha)
\;=\;
\exp\!\left(\frac{-H(\alpha)-C}{\alpha}\right),
\]
where $H(\alpha):=-\alpha\log\alpha-(1-\alpha)\log(1-\alpha)$ is the binary
entropy function.
In particular, for every $P\in\Pc_C$, every measurable set $E\subseteq\X$, and
every $\alpha\in(0,1]$,
\[
P(E)\ge \alpha
\quad\Longrightarrow\quad
P_{\square}(E)\ \ge\ \beta(\alpha).
\]
\end{restatedproposition}

\begin{proof}
Fix $P\in\Pc_C$, a measurable set $E\subseteq\X$, and $\alpha\in(0,1]$ such that
$P(E)\ge \alpha$.
Write
\[
a:=P(E),
\qquad
b:=P_{\square}(E).
\]

Define the measurable map $T:\X\to\{0,1\}$ by
\(
T(x)=\mathbf 1_E(x).
\)
By the data--processing inequality for KL divergence,
\[
\mathrm{TC}(P)
=
\KL(P\,\|\,P_{\square})
\;\ge\;
\KL(P\circ T^{-1}\,\|\,P_{\square}\circ T^{-1}).
\]
The pushforward measures are Bernoulli distributions,
\[
P\circ T^{-1}=\mathrm{Bern}(a),
\qquad
P_{\square}\circ T^{-1}=\mathrm{Bern}(b),
\]
and hence
\[
\mathrm{TC}(P)\ \ge\ d(a\|b),
\]
where
\[
d(a\|b)
=
 a\log\frac{a}{b}
 +(1-a)\log\frac{1-a}{1-b}
\]
is the binary KL divergence.

Since $\mathrm{TC}(P)\le C$, we have $d(a\|b)\le C$.
Expanding,
\[
d(a\|b)
=
-H(a)
-a\log b
-(1-a)\log(1-b),
\]
where $H(a)=-a\log a-(1-a)\log(1-a)$.
Using $\log(1-b)\le 0$, we obtain
\[
d(a\|b)
\;\ge\;
-H(a)-a\log b.
\]
Combining with $d(a\|b)\le C$ yields
\[
\log b
\;\ge\;
-\frac{H(a)+C}{a},
\qquad\text{hence}\qquad
b
\;\ge\;
\exp\!\left(-\frac{H(a)+C}{a}\right).
\]

Define $\phi(x):=\frac{H(x)+C}{x}$ for $x\in(0,1]$.
A direct computation using
$H'(x)=\log\!\bigl(\frac{1-x}{x}\bigr)$ shows that
\[
\phi'(x)
=
\frac{\log(1-x)-C}{x^2}
\;\le\;0,
\]
so $\phi$ is nonincreasing.
Since $a\ge\alpha$, it follows that
\[
\exp\!\left(-\frac{H(a)+C}{a}\right)
\;\ge\;
\exp\!\left(-\frac{H(\alpha)+C}{\alpha}\right)
=
\beta(\alpha).
\]
Therefore $P_{\square}(E)\ge\beta(\alpha)$, as claimed.
\end{proof}

\subsection{Proofs for the permutation-matrix example}\label{sec:proof-permutation}
We collect here the proofs of the claims stated in
Example~\ref{ex:perm-comparison}.
Throughout, let $n\ge 4$, let $\X=[n]\times[n]$, and let
$\F=\{F_\pi:\pi\in S_n\}$ denote the class of permutation graphs,
where
\[
F_\pi := \{(i,\pi(i)):\ i\in[n]\}.
\]

\paragraph{Linear VC dimension.}
We verify that $\LVC(\F)=1$.
Fix an axis--parallel line $L$.

If $L$ is of the form $\{i\}\times[n]$, then for every permutation
$\pi\in S_n$,
\[
F_\pi\cap L = \{(i,\pi(i))\},
\]
which consists of exactly one point.
Hence the induced one--dimensional class $\F|_L$ contains exactly one point from
$L$ for each $\pi$ and therefore has VC dimension equal to $1$.

The same argument applies to lines of the form $[n]\times\{j\}$.
Taking the supremum over all axis--parallel lines yields
\[
\LVC(\F)=1.
\]

\begin{lemma}[Permutation graphs: empirical vs.\ product estimators]
Let $\X=[n]\times[n]$ and $\F=\{F_\pi:\pi\in S_n\}$, and let
$P:=\Unif([n])\otimes\Unif([n])$.
Let $S\sim P^{m}$, let $\widehat P_{\mathrm{emp}}$ denote the empirical
distribution on $\X$, and let $\tilde P_{\prod}$ be the empirical
product-of-marginals estimator.
Then:
\begin{enumerate}
\item If $m\le \tfrac12\sqrt n$, then
\[
\Pr\Bigl[
\sup_{F\in\F}\bigl|\widehat P_{\mathrm{emp}}(F)-P(F)\bigr|\ge \tfrac34
\Bigr]
\ \ge\ \tfrac34.
\]
\item There exists a universal constant $C>0$ such that for every
$\eps,\delta\in(0,1)$, if
\[
 m\ \ge\ C\,\frac{1}{\eps^2}\Bigl(1+\log\tfrac{1}{\delta}\Bigr),
\]
then
\[
\Pr\Bigl[
\sup_{F\in\F}\bigl|\tilde P_{\prod}(F)-P(F)\bigr|\le \eps
\Bigr]
\ \ge\ 1-\delta.
\]
\end{enumerate}
\end{lemma}

\begin{proof}
Let
$S=((I_1,J_1),\ldots,(I_m,J_m))\sim P^{m}$
and denote by
$\widehat P_{\mathrm{emp}}$ the empirical distribution on $\X$.

Define the event
\[
\mathcal E
:=\Bigl\{ I_1,\ldots,I_m \text{ are all distinct and }
          J_1,\ldots,J_m \text{ are all distinct} \Bigr\}.
\]
By a standard birthday bound and a union bound,
\[
\Pr(\mathcal E^c)
\le
2\binom{m}{2}\frac1n
\le \frac{m^2}{n}.
\]
Hence, if $m\le \tfrac12\sqrt n$, then $\Pr(\mathcal E)\ge \tfrac34$.

On the event $\mathcal E$, the sample points form a partial matching in the
complete bipartite graph $[n]\times[n]$.
By Hall’s theorem, such a partial matching can always be extended to a perfect
matching.
Equivalently, there exists a permutation $\pi^*\in S_n$ such that
$(I_t,J_t)\in F_{\pi^*}$ for all $t=1,\ldots,m$.
Consequently,
\[
\widehat P_{\mathrm{emp}}(F_{\pi^*})=1.
\]

On the other hand, for every $\pi\in S_n$,
\[
P(F_\pi)=\sum_{i=1}^n P\bigl((i,\pi(i))\bigr)=\frac1n.
\]
Therefore, on the event $\mathcal E$,
\[
\sup_{F\in\F}\bigl|\widehat P_{\mathrm{emp}}(F)-P(F)\bigr|
\ge
1-\frac1n
\ge \frac34,
\]
where we used $n\ge 4$.
This proves the first claim.

Let $\tilde P_1$ and $\tilde P_2$ denote the empirical marginals of the sample on
$[n]$, and define
\(
\tilde P_{\prod} := \tilde P_1\otimes \tilde P_2.
\)
Since $P$ is a product distribution and $\LVC(\F)=1$,
Remark~\ref{rem:product-high-prob} yields the stated bound for
$\tilde P_{\prod}$.
\end{proof}
\subsection{Infinite product spaces}
\label{sec:infinite-product}

This section provides the technical details underlying the discussion in
Section~\ref{subsubsec:infinite-product}.

\begin{lemma}[Finite $\LVC$ does not guarantee uniform estimability on $\{0,1\}^{\mathbb N}$]
\label{lem:infinite-product-nonestimable}
Let $\X:=\{0,1\}^{\mathbb N}$ equipped with the product $\sigma$--algebra, and let
$\Pc_{\mathrm{prod}}$ be the family of all product probability measures on $\X$.
Define the class of finite--cylinder sets
\[
\F_\infty
\;:=\;
\bigcup_{d\ge 1}
\Bigl\{
  \{x\in\X:\ (x_1,\ldots,x_d)\in A\}
  \;:\;
  A\subseteq \{0,1\}^d
\Bigr\}.
\]
Then:
\begin{enumerate}
\item $\LVC(\F_\infty)=2$.
\item The pair $(\F_\infty,\Pc_{\mathrm{prod}})$ is not uniformly estimable.
More precisely, there exists a universal constant $c>0$ such that for every
$\eps\in(0,1/10)$ and $\delta\in(0,1/3)$, any estimator $\widehat P$ satisfying
\[
\Pr_{S\sim P^{\otimes n}}
\Bigl[
\sup_{F\in\F_\infty}|\widehat P(F)-P(F)| \le \eps
\Bigr]\ \ge\ 1-\delta
\qquad\text{for all }P\in\Pc_{\mathrm{prod}}
\]
must use
\[
n \ \ge\ c\,\frac{d}{\eps}
\]
for arbitrarily large $d$.
\end{enumerate}
\end{lemma}

\begin{proof}
\noindent\emph{Item 1.}
Every axis--parallel line in $\X=\{0,1\}^{\mathbb N}$ consists of exactly two
points, hence $\VC(\G|_L)\le 2$ for any class $\G\subseteq 2^\X$.
Since $\F_\infty$ contains the coordinate cylinder $\{x:x_j=1\}$ for every $j$,
it shatters every such line.
Therefore $\LVC(\F_\infty)=2$.

\medskip
\noindent\emph{Item 2.}
Fix $\eps\in(0,1/10)$ and $\delta\in(0,1/3)$, and suppose for contradiction that
there exist an estimator $\widehat P$ and a finite $n$ such that for every
$P\in\Pc_{\mathrm{prod}}$,
\[
\Pr_{S\sim P^{\otimes n}}
\Bigl[
\sup_{F\in\F_\infty}|\widehat P(F)-P(F)| \le \eps
\Bigr]\ \ge\ 1-\delta.
\]

Let $d\ge 1$ be arbitrary.
For $\theta\in\{\pm1\}^d$, define the product measure
\[
\overline P_\theta
\;:=\;
P_\theta \ \otimes\ \bigotimes_{i>d}\Ber(1/2),
\]
where $P_\theta$ is the $d$--dimensional product distribution from
Proposition~\ref{prop:Omega-d-product}.
Then $\overline P_\theta\in\Pc_{\mathrm{prod}}$.

Consider the subclass of $\F_\infty$ consisting of cylinder sets depending only
on the first $d$ coordinates:
for each $A\subseteq\{0,1\}^d$, define
\[
F_A := \{x\in\X:\ (x_1,\ldots,x_d)\in A\}.
\]
For such sets,
\[
\overline P_\theta(F_A)=P_\theta(A).
\]
Hence the assumed guarantee implies that, for every $\theta$, with probability
at least $1-\delta$,
\[
\sup_{A\subseteq\{0,1\}^d}
\bigl|\widehat P(F_A)-P_\theta(A)\bigr|
\;\le\;\eps.
\]

This yields an estimator achieving uniform $\eps$--accuracy over
$\{0,1\}^d$ under the family $\{P_\theta\}$ with $n$ samples.
By Proposition~\ref{prop:Omega-d-product}, this requires
\[
n \;\ge\; c\,\frac{d}{\eps}.
\]
Since $d$ is arbitrary, no finite sample size can satisfy the assumed uniform
guarantee.
This contradiction completes the proof.
\end{proof}

\subsubsection{Trivial uniform estimation for some VC classes}
We exhibit a class of events with infinite VC dimension for which
\emph{uniform estimation is trivial} under product measures.
While infinite VC dimension can coexist with uniform estimability even in
finite settings, the triviality of estimation (a single sample suffices with
zero error) is a phenomenon specific to infinite product spaces.

\medskip

Let
\[
\X := \{0,1\}^{\mathbb{N}}
\]
be the infinite Boolean cube equipped with its canonical product
$\sigma$--algebra.
Let $\Pc_{\mathrm{prod}}$ denote the family of all product probability measures
on $\X$,
\[
\Pc_{\mathrm{prod}}
=
\left\{
  P = \bigotimes_{i=1}^{\infty} P_i
  \;:\;
  P_i \text{ is a probability measure on } \{0,1\}
\right\}.
\]
For $x\in\X$, write $x_i$ for its $i$-th coordinate.

\medskip

For $\alpha\in[0,1]$, define the tail event
\[
E_\alpha
:=
\left\{
  x\in\X :
  \lim_{n\to\infty}\frac{1}{n}\sum_{i=1}^n x_i = \alpha
\right\},
\]
where the limit is understood whenever it exists.
Each $E_\alpha$ is measurable and invariant under modifications of finitely many
coordinates, and hence is a tail event in the sense of Kolmogorov's $0$-$1$ law.

We now define a class of sets built from these tail events.
Let
\[
\F
:=
\Bigl\{
  F_A := \bigcup_{\alpha\in A} E_\alpha
  \;:\;
  A \subseteq [0,1]\cap\mathbb{Q} \text{ finite}
\Bigr\}.
\]
Note that $\F$ is countable.

\begin{lemma}
\label{lem:tail-vc-infinite}
The class $\F$ has infinite VC dimension.
\end{lemma}

\begin{proof}
Fix $n\in\mathbb{N}$ and choose distinct rationals
$\alpha_1,\ldots,\alpha_n\in(0,1)$.
For each $i\in[n]$, select an arbitrary point
\(
x^{(i)} \in E_{\alpha_i},
\)
which is possible since each $E_{\alpha_i}$ is nonempty.

Let $X := \{x^{(1)},\ldots,x^{(n)}\}\subseteq\X$.
For every subset $S\subseteq[n]$, define
\(
A_S := \{\alpha_i : i\in S\}
\)
and consider the set $F_{A_S}\in\F$.

Since the events $\{E_{\alpha_i}\}_{i=1}^n$ are pairwise disjoint, we have
\[
x^{(i)} \in F_{A_S}
\quad\Longleftrightarrow\quad
\alpha_i\in A_S
\quad\Longleftrightarrow\quad
i\in S.
\]
Thus, for every labeling of the points in $X$, there exists a set
$F_{A_S}\in\F$ that realizes it.
Hence $X$ is shattered by $\F$.
Since $n$ was arbitrary, $\VC(\F)=\infty$.
\end{proof}

\begin{lemma}
\label{lem:tail-01}
For every $P\in\Pc_{\mathrm{prod}}$ and every $\alpha\in[0,1]$,
\[
P(E_\alpha)\in\{0,1\}.
\]
\end{lemma}

\begin{proof}
Recall that a measurable set $A\subseteq\X=\{0,1\}^{\mathbb N}$ is called a
\emph{tail event} if its membership is invariant under modifications of finitely
many coordinates, i.e., if for every $x,y\in\X$,
\[
x_i=y_i \text{ for all but finitely many } i
\quad\Longrightarrow\quad
\mathbf 1_A(x)=\mathbf 1_A(y).
\]

Fix $\alpha\in[0,1]$.
The event
\[
E_\alpha
=
\left\{
x\in\X :
\lim_{n\to\infty}\frac{1}{n}\sum_{i=1}^n x_i = \alpha
\right\}
\]
depends only on the asymptotic frequency of the coordinates.
Changing finitely many coordinates of $x$ does not affect this limit (when it
exists), and therefore $E_\alpha$ is a tail event.

Since $P\in\Pc_{\mathrm{prod}}$ is a product probability measure on $\X$,
Kolmogorov’s $0$--$1$ law applies.
By this law, every tail event under a product measure has probability either
$0$ or $1$ (see, e.g., \cite{shiryaev1996} or
\cite{durrett2019}).
Hence $P(E_\alpha)\in\{0,1\}$, as claimed.
\end{proof}

\begin{proposition}
\label{prop:tail-trivial-estimation}
For every $P\in\Pc_{\mathrm{prod}}$ and every $F\in\F$, we have
\[
P(F)\in\{0,1\}.
\]
Moreover, uniform estimation over $\F$ under $\Pc_{\mathrm{prod}}$ is trivial:
with probability one over a single draw $X\sim P$,
\[
\sup_{F\in\F}\bigl|\widehat P_1(F)-P(F)\bigr| = 0,
\]
where $\widehat P_1(F):=\mathbf 1\{X\in F\}$.
\end{proposition}

\begin{proof}
Let $F=F_A$ for some finite $A\subseteq[0,1]\cap\mathbb{Q}$.
Since the events $\{E_\alpha\}_{\alpha\in A}$ are disjoint,
\[
P(F) = \sum_{\alpha\in A} P(E_\alpha).
\]
By Lemma~\ref{lem:tail-01}, each summand belongs to $\{0,1\}$, and hence so does
$P(F)$.

Now fix $X\sim P$.
Almost surely, $X$ belongs to at most one of the sets $E_\alpha$.
Consequently, for every $F\in\F$, the value $P(F)$ is determined exactly by the
membership of $X$ in $F$.
Thus $\widehat P_1(F)=P(F)$ for all $F\in\F$ almost surely.
\end{proof}

\medskip

Proposition~\ref{prop:tail-trivial-estimation} shows that, in infinite product
spaces, strong independence can render uniform estimation \emph{trivial} even
for classes with infinite VC dimension.
The novelty of this example lies not merely in the coexistence of infinite VC
dimension and estimability, but in the fact that \emph{exact} estimation is
possible from a single sample with probability one.

This phenomenon is inherently infinite-dimensional and has no analogue in
finite product spaces.
In particular, it demonstrates that neither VC dimension nor linear VC
dimension can characterize uniform estimability or its rates in the
infinite-product regime.
\subsection{Example: Necessity of d dependence for lower bound}
\begin{lemma}[Fano via KL from the mean distribution]
\label{lem:fano-kl-mean}
Let $M\ge 2$ and let $V\sim\Unif([M])$.
For each $v\in[M]$ let $P_v$ be a distribution on $\Y$, and draw
$Y\mid(V=v)\sim P_v$.
Let $\widehat V=f(Y)$ be any estimator and set $P_e:=\Pr(\widehat V\neq V)$.
Define the mean (mixture) distribution
\[
\bar P \ :=\ \frac1M\sum_{v=1}^M P_v.
\]
Then
\begin{equation}
\log M - \frac1M\sum_{v=1}^M \KL(P_v\|\bar P)
\ \le\ h(P_e)+P_e\log(M-1),
\label{eq:fano-kl-mean-tight}
\end{equation}
where $h(p)=-p\log p-(1-p)\log(1-p)$ is the binary entropy.

In particular,
\begin{equation}
P_e
\ \ge\
1-\frac{\frac1M\sum_{v=1}^M \KL(P_v\|\bar P)+\log 2}{\log M}.
\label{eq:fano-kl-mean-simplified}
\end{equation}
\end{lemma}

\begin{proof}
Apply the  Fano inequality to $X:=V$ (with $|\X|=M$), $Y$, and
$\tilde X:=\widehat V=f(Y)$:
\begin{equation}
H(V\mid Y)\ \le\ h(P_e)+P_e\log(M-1).
\label{eq:wiki-fano-again}
\end{equation}

We now compute $H(V\mid Y)$.
Let $p(v,y)$ be the joint law of $(V,Y)$. Since $V$ is uniform and
$Y\mid(V=v)\sim P_v$,
\[
p(v,y)=\frac1M P_v(y),
\qquad
p_Y(y)=\sum_{u=1}^M p(u,y)=\bar P(y).
\]
Hence, by Bayes' rule,
\[
p(v\mid y)=\frac{p(v,y)}{p_Y(y)}=\frac{P_v(y)}{M\,\bar P(y)}.
\]
Using the definition of conditional entropy,
\begin{align*}
H(V\mid Y)
&= -\sum_{v=1}^M \int p(v,y)\,\log p(v\mid y)\,dy \\
&= -\sum_{v=1}^M \int \frac1M P_v(y)\,
\log\!\Bigl(\frac{P_v(y)}{M\,\bar P(y)}\Bigr)\,dy \\
&= \log M - \frac1M\sum_{v=1}^M \int P_v(y)\log\!\Bigl(\frac{P_v(y)}{\bar P(y)}\Bigr)\,dy \\
&= \log M - \frac1M\sum_{v=1}^M \KL(P_v\|\bar P).
\end{align*}
Plug this identity into \eqref{eq:wiki-fano-again} to obtain
\eqref{eq:fano-kl-mean-tight}. Finally, using $h(P_e)\le \log 2$ and
$\log(M-1)\le \log M$ yields \eqref{eq:fano-kl-mean-simplified}.
\end{proof}

\begin{lemma}[KL for biased product Bernoullis]
\label{lem:kl-product-bernoulli}
Fix $d\in\mathbb N$ and $\nu\in(0,\tfrac14)$.
For each $\theta\in\{\pm1\}^d$ define the product distribution
\[
P_\theta \;:=\; \bigotimes_{i=1}^d \Ber\!\left(\tfrac12+\theta_i\nu\right).
\]
Then for any $\theta,\theta'\in\{\pm1\}^d$,
\[
\KL(P_\theta\|P_{\theta'})
\;=\;
\Delta(\theta,\theta')\cdot D_\nu,
\]
where $\Delta(\theta,\theta'):=|\{i:\theta_i\neq\theta_i'\}|$ is the Hamming distance and
\[
D_\nu
\;:=\;
\KL\!\left(\Ber\!\left(\tfrac12+\nu\right)\Big\|\Ber\!\left(\tfrac12-\nu\right)\right)
=
2\nu\log\!\left(\frac{\tfrac12+\nu}{\tfrac12-\nu}\right)
=
2\nu\log\!\left(\frac{1+2\nu}{1-2\nu}\right).
\]
Moreover, for all $\nu\in(0,\tfrac14)$,
\[
8\nu^2 \;\le\; D_\nu \;\le\; \frac{32}{3}\,\nu^2.
\]
\end{lemma}

\begin{proof}
By additivity of KL divergence for product measures,
\[
\KL(P_\theta\|P_{\theta'})
=
\sum_{i=1}^d
\KL\!\left(\Ber\!\left(\tfrac12+\theta_i\nu\right)\Big\|\Ber\!\left(\tfrac12+\theta_i'\nu\right)\right).
\]
If $\theta_i=\theta_i'$, the $i$th summand is $0$.
If $\theta_i\neq\theta_i'$, then $\{\theta_i,\theta_i'\}=\{+1,-1\}$ and the $i$th summand equals
$\KL(\Ber(\tfrac12+\nu)\|\Ber(\tfrac12-\nu))=:D_\nu$.
Hence the sum has exactly $\Delta(\theta,\theta')$ nonzero terms, proving
\[
\KL(P_\theta\|P_{\theta'})
=
\Delta(\theta,\theta')\,D_\nu.
\]

Let $p=\tfrac12+\nu$ and $q=\tfrac12-\nu$. Then
\begin{align*}
D_\nu
&=
p\log\frac{p}{q} + (1-p)\log\frac{1-p}{1-q}
=
p\log\frac{p}{q} + q\log\frac{q}{p}\\
&=
(p-q)\log\frac{p}{q}
=
2\nu\log\!\left(\frac{\tfrac12+\nu}{\tfrac12-\nu}\right)
=
2\nu\log\!\left(\frac{1+2\nu}{1-2\nu}\right).
\end{align*}

Set $x:=2\nu\in(0,\tfrac12)$.
For $x\in(0,1)$ one has
\[
2x
\;\le\;
\log\!\left(\frac{1+x}{1-x}\right)
\;\le\;
\frac{2x}{1-x^2}.
\]

Applying these inequalities,
\[
D_\nu
=
2\nu\log\!\left(\frac{1+x}{1-x}\right)
\ge
2\nu\cdot 2x
=
8\nu^2,
\]
and
\[
D_\nu
\le
2\nu\cdot \frac{2x}{1-x^2}
=
\frac{8\nu^2}{1-4\nu^2}
\le
\frac{32}{3}\nu^2,
\]
since $\nu<1/4$ implies $1-4\nu^2\ge 3/4$.
\end{proof}

\begin{lemma}[Hellinger separation for biased product Bernoullis]
\label{lem:hellinger-product-bernoulli}
Let $k\ge 1$ and define product measures on $\{0,1\}^k$ by
\[
P_+ := \Ber\!\left(\tfrac12+\nu\right)^{\otimes k},
\qquad
P_- := \Ber\!\left(\tfrac12-\nu\right)^{\otimes k},
\]
where $0<\nu<1/2$. Let
\[
H^2(P,Q)\ :=\ 2\bigl(1-\rho(P,Q)\bigr),
\qquad
\rho(P,Q)\ :=\ \sum_{x}\sqrt{P(x)Q(x)}
\]
denote the squared Hellinger distance and the Bhattacharyya coefficient.
Then for every $0<\eps\le 1$, if
\[
\nu \;\ge\; \sqrt{\frac{\eps}{2k}},
\]
we have
\[
H^2(P_+,P_-) \;\ge\; \eps .
\]
\end{lemma}

\begin{proof}
For a single coordinate, let
\[
P_+^{(1)}=\Ber\!\left(\tfrac12+\nu\right),
\qquad
P_-^{(1)}=\Ber\!\left(\tfrac12-\nu\right).
\]
A direct computation gives
\[
\rho\!\left(P_+^{(1)},P_-^{(1)}\right)
= \sqrt{1-4\nu^2}.
\]

For product measures $P=\bigotimes_{i=1}^k P_i$ and $Q=\bigotimes_{i=1}^k Q_i$
on a finite product space, one has the multiplicativity
\[
\rho(P,Q)=\prod_{i=1}^k \rho(P_i,Q_i).
\]
Applying this with $P_i=P_+^{(1)}$ and $Q_i=P_-^{(1)}$ yields
\[
\rho(P_+,P_-)
= \bigl(\rho(P_+^{(1)},P_-^{(1)})\bigr)^k
= (1-4\nu^2)^{k/2},
\]
and therefore
\begin{equation}
H^2(P_+,P_-)
= 2\Bigl(1-(1-4\nu^2)^{k/2}\Bigr).
\label{eq:exact-hellinger}
\end{equation}

Since $4\nu^2\le 1$, we may use $(1-x)^m\le e^{-mx}$ for $x\in[0,1]$ to obtain
\[
(1-4\nu^2)^{k/2}\le e^{-2k\nu^2}.
\]
Substituting into~\eqref{eq:exact-hellinger} gives
\[
H^2(P_+,P_-)\ge 2\bigl(1-e^{-2k\nu^2}\bigr).
\]

If $\nu^2\ge \eps/(2k)$, then $2k\nu^2\ge \eps$, hence
\[
H^2(P_+,P_-)\ge 2\bigl(1-e^{-\eps}\bigr).
\]
For $\eps\in(0,1]$ we have $1-e^{-\eps}\ge \eps/2$, and therefore
\[
H^2(P_+,P_-)\ge 2\cdot \frac{\eps}{2}=\eps,
\]
as claimed.
\end{proof}

\begin{proposition}[$\Omega(d)$ samples for TV/uniform estimation over the $d$--cube]
\label{prop:Omega-d-product}
Let $\X=\{0,1\}^d$ and let $\F=2^\X$.
Fix $\eps\in(0,1/10)$ and $\delta\in(0,1/3)$.
Consider the family of \emph{product} distributions
\[
\Pc := \{P_\theta : \theta\in\{\pm1\}^d\},
\qquad
P_\theta := \bigotimes_{i=1}^d \Ber\!\left(\tfrac12+\theta_i \nu\right),
\]
where
\[
\nu \;:=\; 4\,\sqrt{\frac{\eps}{d}}.
\]
Assume $d$ is large enough so that $\nu<1/4$.

Suppose an estimator $\widehat P$ based on $n$ i.i.d.\ samples satisfies
\[
\Pr_{S\sim P^{\otimes n}}
\Bigl[
\sup_{F\in\F}|\widehat P(F)-P(F)| \le \eps
\Bigr]\ \ge\ 1-\delta
\qquad\text{for all }P\in\Pc.
\]
Then necessarily
\[
n \ \ge\ c\,\frac{d}{\eps},
\]
for a universal constant $c>0$.
In particular, for constant $\eps$, one must have $n=\Omega(d)$.
\end{proposition}

\begin{proof}
\medskip
\noindent\textbf{Step 1: a large code with Hamming distance $\ge d/4$.}
Let $\Theta\subseteq\{\pm1\}^d$ be such that for all distinct
$\theta,\theta'\in\Theta$,
\[
\Delta(\theta,\theta'):=|\{i:\theta_i\neq\theta_i'\}|\ \ge\ d/4,
\]
and $|\Theta|\ge 2^{c_1 d}$ for a universal constant $c_1>0$.
Such a set exists by the Gilbert--Varshamov bound.
We restrict attention to the subfamily $\{P_\theta:\theta\in\Theta\}$.

\medskip
\noindent\textbf{Step 2: separation in total variation.}
Fix distinct $\theta,\theta'\in\Theta$ and let $k:=\Delta(\theta,\theta')\ge d/4$.
Let $I:=\{i\in[d]:\theta_i\neq\theta_i'\}$, so $|I|=k$.
Since the Bhattacharyya coefficient factors over product measures,
and the marginals coincide on $[d]\setminus I$, we have
\[
H^2(P_\theta,P_{\theta'}) \;=\; H^2\!\left(\Ber(\tfrac12+\nu)^{\otimes k},
\Ber(\tfrac12-\nu)^{\otimes k}\right).
\]
By Lemma~\ref{lem:hellinger-product-bernoulli} applied with $\eps' := 8\eps$,
it suffices that
\[
\nu \ \ge\ \sqrt{\frac{\eps'}{2k}}
\ =\ \sqrt{\frac{8\eps}{2k}}
\ =\ 2\sqrt{\frac{\eps}{k}}
\ \le\ 4\sqrt{\frac{\eps}{d}},
\]
where the last inequality uses $k\ge d/4$.
With our choice $\nu = 4\sqrt{\eps/d}$, we obtain
\[
H^2(P_\theta,P_{\theta'}) \ \ge\ 8\eps .
\]
Since $\TV(P,Q)\ge \tfrac12 H^2(P,Q)$, it follows that
\[
\TV(P_\theta,P_{\theta'}) \ \ge\ 4\eps
\qquad
\text{for all distinct }\theta,\theta'\in\Theta.
\]

\medskip
\noindent\textbf{Step 3: uniform estimation implies decoding.}
For distributions on a finite domain,
\[
\TV(P,Q)=\sup_{F\subseteq\X}|P(F)-Q(F)|.
\]
Thus, by assumption, with probability at least $1-\delta$,
\[
\TV(\widehat P,P_\theta)\ \le\ \eps.
\]
Define the decoder
\[
\widehat\theta
\in \arg\min_{\theta'\in\Theta} \TV(\widehat P,P_{\theta'}).
\]
On the event $\TV(\widehat P,P_\theta)\le\eps$, separation implies
$\widehat\theta=\theta$: indeed, for $\theta'\neq\theta$,
\[
\TV(\widehat P,P_{\theta'})
\ \ge\ \TV(P_\theta,P_{\theta'})-\TV(\widehat P,P_\theta)
\ \ge\ 4\eps-\eps
\ >\ \eps.
\]
Hence
\begin{equation}\label{eq:decode-success}
\Pr[\widehat\theta=\theta]\ \ge\ 1-\delta.
\end{equation}

\medskip
\noindent\textbf{Step 4: Fano via KL from the mean distribution.}
Let $V$ be uniform on $\Theta$, and let
$S\mid(V=\theta)\sim(P_\theta)^{\otimes n}$.
Let $\widehat V=\widehat\theta(S)$.
By~\eqref{eq:decode-success},
\[
P_e:=\Pr[\widehat V\neq V]\ \le\ \delta.
\]
Let
\[
\bar Q \ :=\ \frac1{|\Theta|}\sum_{\theta\in\Theta}(P_\theta)^{\otimes n}
\]
be the mean distribution of $S$.
Applying Lemma~\ref{lem:fano-kl-mean} gives
\begin{equation}\label{eq:fano-mean}
\log|\Theta|-\frac1{|\Theta|}\sum_{\theta\in\Theta}
\KL\!\left((P_\theta)^{\otimes n}\middle\|\bar Q\right)
\ \le\ h(P_e)+P_e\log(|\Theta|-1).
\end{equation}
Using $P_e\le\delta$, $h(P_e)\le\log2$, and $\log(|\Theta|-1)\le\log|\Theta|$,
we obtain
\[
\frac1{|\Theta|}\sum_{\theta\in\Theta}
\KL\!\left((P_\theta)^{\otimes n}\middle\|\bar Q\right)
\ \ge\ (1-\delta)\log|\Theta|-\log2.
\]

\medskip
\noindent\textbf{Step 5: upper bound the KL from the mean.}
By convexity of KL in its second argument,
\[
\KL\!\left((P_\theta)^{\otimes n}\middle\|\bar Q\right)
\le
\frac1{|\Theta|}\sum_{\theta'\in\Theta}
\KL\!\left((P_\theta)^{\otimes n}\middle\|(P_{\theta'})^{\otimes n}\right).
\]
Averaging over $\theta$ yields
\[
\frac1{|\Theta|}\sum_{\theta\in\Theta}
\KL\!\left((P_\theta)^{\otimes n}\middle\|\bar Q\right)
\le
\frac{n}{|\Theta|^2}\sum_{\theta,\theta'\in\Theta}
\KL(P_\theta\|P_{\theta'}).
\]
By Lemma~\ref{lem:kl-product-bernoulli},
\[
\KL(P_\theta\|P_{\theta'})
=\Delta(\theta,\theta')\,D_\nu
\le d\cdot D_\nu.
\]
Using the upper bound $D_\nu\le \frac{32}{3}\nu^2$, we get
\[
\frac1{|\Theta|}\sum_{\theta\in\Theta}
\KL\!\left((P_\theta)^{\otimes n}\middle\|\bar Q\right)
\ \le\ n\cdot C\,d\,\nu^2
\]
for a universal constant $C>0$.

\medskip
\noindent\textbf{Step 6: conclude the sample complexity lower bound.}
Combining the lower bound from Step~4 with the upper bound from Step~5 and using
$\log|\Theta|\ge c_1 d$, we obtain
\[
n\cdot C\,d\,\nu^2
\;\ge\;
(1-\delta)\log|\Theta|-\log 2
\;\ge\;
(1-\delta)c_1 d-\log 2 .
\]
Since $\delta\le 1/3$, the right-hand side is at least $c_2 d$ for a universal
constant $c_2>0$ (absorbing the additive $-\log 2$ term into the constant for
all sufficiently large $d$). Hence
\[
n\cdot C\,d\,\nu^2 \;\ge\; c_2 d,
\]
and therefore
\[
n
\;\ge\;
\frac{c_2}{C\,\nu^2}.
\]
Recalling that $\nu = 4\sqrt{\eps/d}$ (so that $\nu^2 = 16\,\eps/d$), this yields
\[
n
\;\ge\;
\frac{c_2}{C}\cdot \frac{d}{16\,\eps}
\;=\;
\frac{c_2}{16C}\,\frac{d}{\eps}.
\]
Setting $c:=c_2/(16C)$ completes the proof.
\end{proof}

\end{document}